\documentclass{article} 
\usepackage{iclr2025_conference,times}
\iclrfinalcopy

\usepackage{amsmath,amsfonts,bm}

\def\eqref#1{equation~\ref{#1}}

\def\1{\bm{1}}

\DeclareMathAlphabet{\mathsfit}{\encodingdefault}{\sfdefault}{m}{sl}
\SetMathAlphabet{\mathsfit}{bold}{\encodingdefault}{\sfdefault}{bx}{n}

\usepackage[utf8]{inputenc} 
\usepackage[T1]{fontenc}    
\usepackage{hyperref}       
\usepackage{url}            
\usepackage{booktabs}       
\usepackage{amsfonts}       
\usepackage{nicefrac}       
\usepackage{microtype}      
\usepackage{xcolor}         
\usepackage{multirow}      
\usepackage{float}

\usepackage{amsmath}
\usepackage{amssymb}
\usepackage{mathtools}
\usepackage{amsthm}
\usepackage{graphicx}
\usepackage{wrapfig}

\theoremstyle{plain}
\newtheorem{theorem}{Theorem}[section]
\newtheorem{proposition}[theorem]{Proposition}
\newtheorem{lemma}[theorem]{Lemma}

\theoremstyle{definition}
\newtheorem{definition}[theorem]{Definition}
\newtheorem{assumption}[theorem]{Assumption}
\theoremstyle{remark}

\usepackage[most]{tcolorbox}
\usepackage{xcolor}
\usepackage{subcaption}
\usepackage{makecell}

\definecolor{myblue}{RGB}{25, 45, 120}
\definecolor{mygray}{RGB}{240, 240, 245}

\newtcolorbox{keytakeaway}[1][]{
  colback=mygray,        
  colframe=myblue,       
  boxrule=1.2pt,         
  arc=6pt,               
  left=10pt,
  right=10pt,
  top=8pt,
  bottom=8pt,
  fonttitle=\bfseries,
  coltitle=black,        
  attach title to upper, 
  title={#1}             
}

\title{Massive Activation Gating Channel in Large Language Models}

\author{Minjia Mao$^1$, Shi Chen$^1$, Bowen Yin$^2$, Xiao Fang$^1$ \\
$^1$ University of Delaware \\
$^2$ Peking University \\
\texttt{\{mjmao,shichen,xfang\}@udel.edu} \\
\texttt{bowenyin@stu.pku.edu.cn} 
}

\begin{document}

\maketitle

\begin{abstract}
Massive activations, a phenomenon in which a small number of hidden channels exhibit exceptionally large magnitudes, are pervasive in large language models (LLMs). However, the mechanism by which a token develops massive activations as it propagates through a pretrained LLM remains poorly understood. In this paper, we find that the emergence of massive activations is controlled by a single channel in the input embedding to a spike feed-forward network (FFN). The position of this channel is fixed for a particular LLM. We name this channel the massive activation gating channel (MAGC).   When the value of the MAGC is sufficiently large (or small, depending on the LLM), the output of the spike FFN exhibits massive activations. Examining six LLMs across four model families and different model sizes, we verify the existence and effect of MAGC. We further provide a theoretical explanation of the mechanism by which MAGC induces massive activations. When the value of MAGC is sufficiently large (or small), the output of a spike FFN asymptotically reduces to a quadratic form that mixes a few columns of the down-projection matrix of the FFN. Since these columns exhibit the shape of massive activations, the output therefore exhibits massive activations.

\end{abstract}

\section{Introduction}

Large Language Models (LLMs) have demonstrated strong capabilities on various tasks \citep{bubeck2023sparks}. While the majority of existing work focuses on external behaviors of LLMs, it is also important to understand the internal mechanisms of LLMs. Among them, attention sink \citep{xiao2023efficient} and massive activations \citep{sun2024massive} have been found in modern LLMs across different model families. Specifically, attention sink refers to the phenomenon in which the first token (or other tokens with little semantic importance) in a sequence receives a disproportionately large amount of attention from subsequent tokens \citep{yu2024unveiling,barbero2025llms}, typically in the middle transformer blocks. The cause of attention sink is usually attributed to massive activations \citep{sun2024massive,yan2024unveiling}, where a few channels (dimensions) in the hidden representation of the first token (or other tokens with little semantic importance) exhibit large activation magnitudes before attention computation. 
It is important to unveil the mechanism underlying massive activations and attention sink, as these phenomena have been widely employed in various critical tasks, including KV-cache optimization \citep{ge2023model}, inference \citep{xiao2023efficient}, quantization \citep{son2024prefixing}, and hallucination mitigation \citep{liu2026sinktrack}.

In this study, we focus on massive activations. Although existing research has extensively demonstrated the phenomenon of massive activations, research on the mechanism underlying massive activations remains limited. We address this problem by asking: \textit{how does a token develop massive activations as it propagates a pretrained transformer?} We refer to tokens that exhibit massive activations as spike tokens \citep{sun2026spike}.

As LLMs are powered by billions of parameters, our analysis draws attention to the computation of each sub-module (attention and feed-forward networks) in transformers. 
Existing work \citep{zhang2026attention,sun2026spike} has demonstrated that massive activations emerge after certain feed-forward networks (FFNs) and are maintained in the middle transformer blocks through the residual layer (see Section \ref{sec:architecture} for more details). For example, in LLaMA-2-7B, massive activations of a spike token emerge after the FFN in the second transformer block \citep{sun2024massive,sun2026spike} and are maintained until reaching the last FFN. We refer to the FFN that produces massive activations as the spike FFN. 
The magnitudes of massive activations remain similar in the middle transformer blocks, indicating that only the residual layer maintains massive activations and that no further massive activations are triggered during the middle computation.

\citet{sun2026spike} assume that the SiLU activation function in the spike FFN behaves as a near-identity function for spike tokens. Under this assumption, each output channel $k$ of the spike FFN can be expressed as a quadratic form characterized by a channel-specific coefficient matrix $S_k$. The magnitude of output channel $k$ is determined by the dominant eigenvalue of $S_k$. 
Since the dominant eigenvalues associated with a few channels are substantially larger than those of other channels, massive activation emerges.

However, our empirical results show that the near-identity assumption on which \citet{sun2026spike} rely to explain massive activations is only valid for a small number of models. Among the six models examined in this study, only one model approximately satisfies the assumption (see Appendix \ref{appendix:violation} for more details). 
Moreover, their explanation overlooks the properties of input embeddings to the spike FFN. 
Instead, we find that the emergence of massive activations is controlled by a single channel in the input embeddings of a spike token to the spike FFN: massive activations appear once the value of this channel is sufficiently large (or small, depending on the LLM). 
For a given LLM, the position of this channel is fixed.
We refer to this channel as the massive activation gating channel (MAGC).
Unlike massive activations, whose magnitudes can reach approximately 400 in LLaMA-3.2-1B, the MAGC in the embedding of a spike token has an average value of only around 8.8. Although this value is substantially larger than those of the other channels in the embedding, its magnitude remains moderate because the embedding is layer-normalized before being fed into the FFN.


We verify this finding across multiple decoder-only LLM families, including LLaMA, Qwen, DeepSeek, and Mistral, with model sizes ranging from 1B to 32B parameters. 
Moreover, our theoretical analysis shows that massive activations can be attributed to the alignment between MAGC and certain rows and columns of the spike FFN weight matrices. 
Similarly, we empirically demonstrate that this alignment occurs in different LLMs. We summarize our contributions as follows. 

    
1. The emergence of massive activations is controlled by a single channel in the input embedding to a spike FFN, and the position of this channel is fixed for a particular LLM. We name this channel the massive activation gating channel (MAGC). 
    When the value of the MAGC is sufficiently large (or small, depending on the LLM), the output of the spike FFN exhibits massive activations. Examining six LLMs across four model families and different model sizes, we verify the existence and effect of MAGC.
    
    
2. We provide a theoretical explanation of the mechanism by which MAGC induces massive activations. When the value of MAGC is sufficiently large (or small), the output of a spike FFN asymptotically reduces to a quadratic form that mixes a few columns of the down-projection matrix of the FFN. Since these columns exhibit the shape of massive activations, the output therefore exhibits massive activations.


\section{Transformer Architecture} 
\label{sec:architecture}

A dense transformer model consists of $L$ transformer blocks, each of which can be further decomposed into modules. Taking LLaMA as an example \citep{touvron2023llama}, denote the input to the $l$-th transformer block as $X^l=\langle \boldsymbol{x}_0^l,\boldsymbol{x}_1^l,\cdots,\boldsymbol{x}_T^l \rangle$. Without further specification, $\cdot_t^l$ denotes an embedding of the $t$-th token in the $l$-th transformer block. Since our following analysis focuses on the single spike FFN, we omit the superscript $l$ for notational simplicity. Each block of LLaMA consists of an attention sub-module and an FFN sub-module, each of which uses pre-normalization (RMSNorm \citep{zhang2019root}) and includes a residual layer after its computation. 

The input embedding is first normalized by an RMSNorm, referred to as input layer norm. Formally, if $\boldsymbol{x}_t=\langle x_t(0),x_t(1),\cdots, x_t(d-1) \rangle$, where $d$ is the dimension of the embedding, then 
\begin{equation}
\label{eq:paln}
    \boldsymbol{n}_t = \text{RMSNorm}(\boldsymbol{x}_t): n_t(j) = \frac{x_t(j)\cdot w_j}{\text{RMS}(\boldsymbol{x}_t)} , \text{RMS}(\boldsymbol{x}_t)=\sqrt{\frac{\sum_j x_t(j)^2}{d} + \epsilon},
\end{equation}
where $w_j$ are trainable parameters and $\epsilon$ is a predefined small number. We denote $\cdot(j)$ as the $j$-th channel of the embedding. Next, the normalized embedding $\boldsymbol{n}_t$ is transformed by multi-head attention (MHA, or more specifically, grouped-query attention \citep{ainslie2023gqa}):
\begin{equation*}
    (\boldsymbol{u}_0, \boldsymbol{u}_1, \cdots, \boldsymbol{u}_T) = \text{MHA} (\boldsymbol{n}_0, \boldsymbol{n}_1, \cdots, \boldsymbol{n}_T),
\end{equation*}
where $\boldsymbol{u}_t$ is the embedding output of MHA. Usually, a transformer includes a residual layer after the multi-head attention \citep{vaswani2017attention}.
Next, the FFN sub-module follows a similar configuration with another RMSNorm of different parameters. 
\begin{equation*}
    \boldsymbol{h}_t = \text{RMSNorm} (\boldsymbol{x}_t + \boldsymbol{u}_t). 
\end{equation*}
A feed-forward network (FFN) is then applied to $\boldsymbol{h}_t$. In LLaMA, the FFN consists of a SwiGLU gated activation function \citep{shazeer2020glu} with three weight matrices: 
\begin{equation}
\label{eq:ffn}
    \boldsymbol{y}_t=\text{FFN} (\boldsymbol{h}_t), \boldsymbol{y}_t=\boldsymbol{W}_{\text{down}} \cdot \left( \text{SiLU}\left(\boldsymbol{W}_{\text{gate}}\boldsymbol{h}_t \right) \odot \left(\boldsymbol{W}_{\text{up}}\boldsymbol{h}_t \right) \right),
\end{equation}
where $\odot$ is the element-wise product and $\boldsymbol{y}$ is the output of FFN. $\boldsymbol{W}_{\text{gate}}, \boldsymbol{W}_{\text{up}}\in \mathbb{R}^{m\times d}$ and $\boldsymbol{W}_{\text{down}} \in \mathbb{R}^{d\times m}$ are pretrained weight matrices, where $m>d$ is the FFN intermediate dimension. 
Finally, another residual layer follows the FFN. That is, the output of a transformer block is $ \boldsymbol{x}_t + \boldsymbol{u}_t + \boldsymbol{y}_t$.

\begin{figure*}[htb]
    \centering
    \begin{subfigure}[b]{0.48\textwidth}
        \centering
        \includegraphics[width=\linewidth]{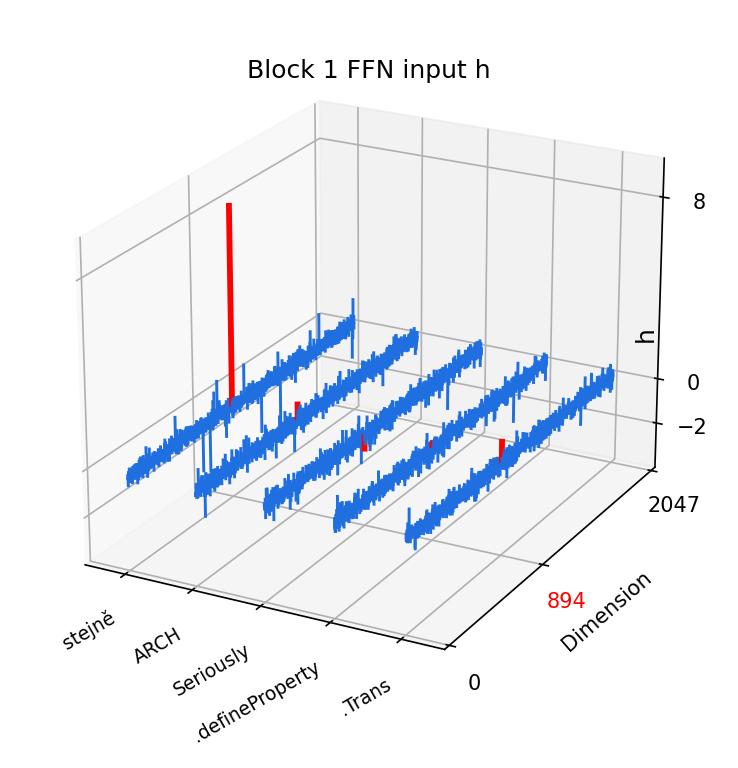}
        \caption{Value of input $\boldsymbol{h}$ to the spike FFN. We highlight the massive activation gating channel (894) in red.}
        \label{fig:llama_token0_input}
    \end{subfigure}
    \hfill
    \begin{subfigure}[b]{0.48\textwidth}
        \centering
        \includegraphics[width=\linewidth]{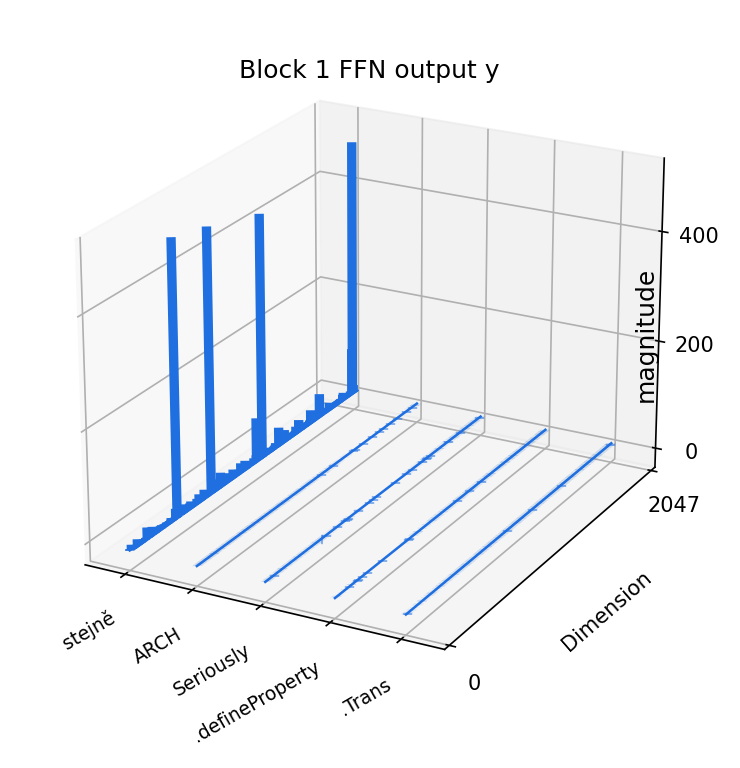}
        \caption{Magnitude of output $\boldsymbol{y}$ of the spike FFN. Token 0 exhibits massive activations. We depict magnitudes instead of actual values for better visualization.}
        \label{fig:llama_token0_output}
    \end{subfigure}

    \caption{Inputs and Outputs of the Spike FFN.}
    \label{fig:llama_token}
\end{figure*}

\section{Massive Activation Gating Channel}
\label{sec:3}

\textbf{Experimental Setup.} In this section, we utilize a small but advanced LLM, LLaMA-3.2-1B \citep{dubey2024llama}. We verify the findings discovered in this section on other model families with larger model sizes in Section~\ref{sec:llm_extension}. In LLaMA-3.2-1B, $d=2,048$ and $m=8,192$. 
We generate token sequences (without the BOS token) completely randomly to avoid the influence of semantics. Specifically, we randomly sample five tokens in the vocabulary set $\mathcal{V}$ and construct an input token sequence to the LLM. 
Note that the input tokens can be multilingual as well as completely meaningless. We conduct this sampling 5 times. 

\subsection{Existence of MAGC in the Input to Spike FFN}

Following \citet{sun2026spike}, we demonstrate the top-3 channel magnitudes of the transformer block outputs $\boldsymbol{x}+\boldsymbol{u}+\boldsymbol{y}$ in Figure \ref{fig:llama1_top3} in Appendix \ref{appendix:exp_llama1}. Similarly, we show that massive activations are first triggered by the FFN in an early layer, and are then preserved and propagated along the residual stream across layers. As depicted, we identify that in LLaMA-3.2-1B, the FFN in the second transformer block (i.e., block 1) produces massive activations. We name this FFN the spike FFN and focus our analysis on it.

A straightforward method for analyzing the causes of massive activations is to trace their inputs. To this end, we directly visualize the embeddings of the inputs and outputs of the spike FFN for different tokens in Figure \ref{fig:llama_token}, using a token sequence example. We demonstrate other examples in Figure \ref{fig:llama_token_other} in Appendix \ref{appendix:other_sample}. 
As shown in Figure \ref{fig:llama_token0_output}, the first token (``stejně'') exhibits massive activations, whereas other tokens do not. Comparing their inputs to the FFN, we find that the input $h$ for the first token has a single outlier value, as shown in Figure \ref{fig:llama_token0_input}.
Specifically, the 894th channel of the first token has a value of 8.807, while the absolute values of the other channels are all less than 2. The 894th channel of the other tokens has values of 0.814, -0.763, 0.335, and 1.029, respectively.
In the next subsection, we will show that this single channel controls the emergence of massive activations, and we name this channel the massive activation gating channel (MAGC).



\begin{figure*}[htb]
    \centering
    \hfill
    \begin{subfigure}[b]{0.48\textwidth}
        \centering
        \includegraphics[width=\linewidth]{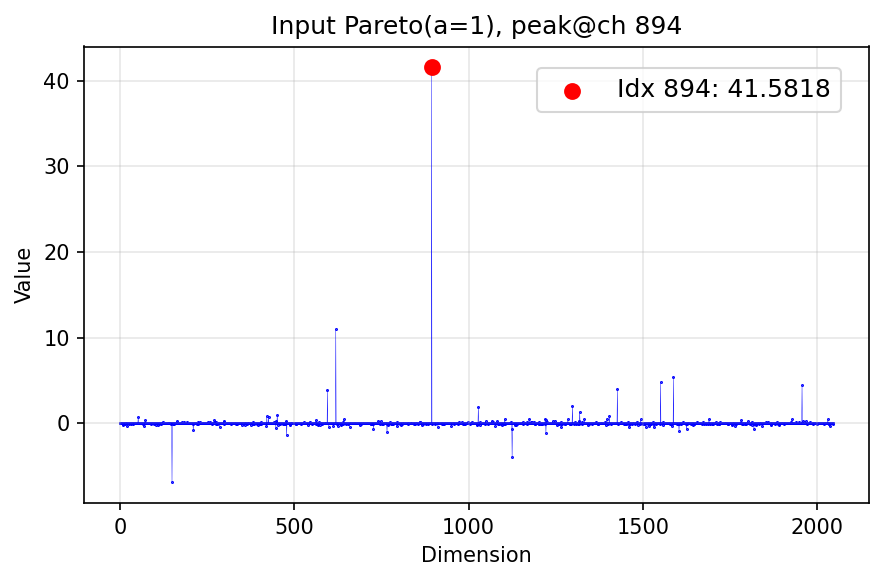}
    \end{subfigure}
    \hfill
    \begin{subfigure}[b]{0.48\textwidth}
        \centering
        \includegraphics[width=\linewidth]{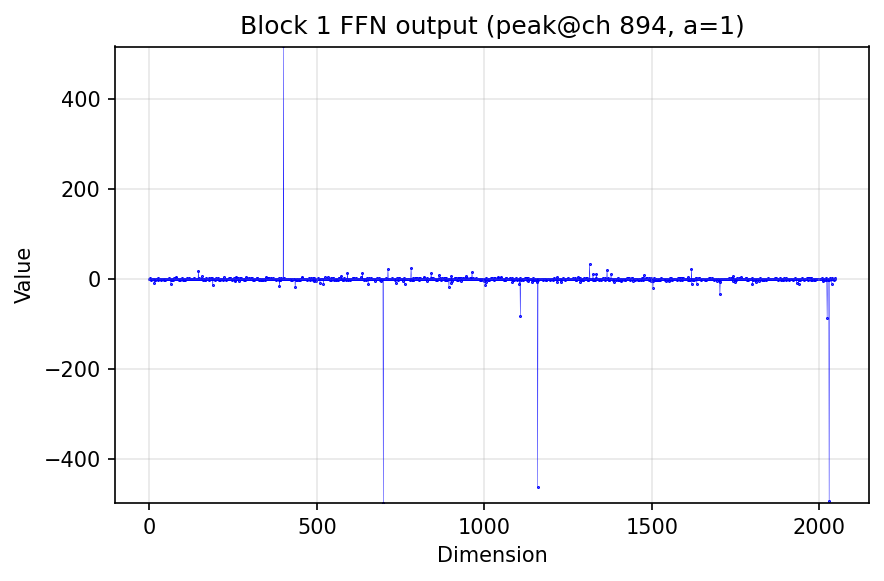}
    \end{subfigure}
    
    \hfill
    \begin{subfigure}[b]{0.48\textwidth}
        \centering
        \includegraphics[width=\linewidth]{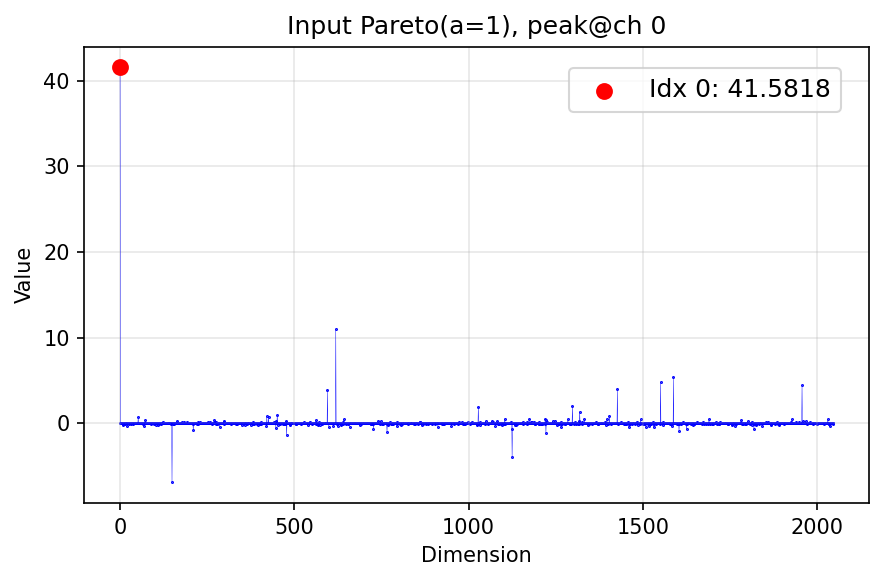}
    \end{subfigure}
    \hfill
    \begin{subfigure}[b]{0.48\textwidth}
        \centering
        \includegraphics[width=\linewidth]{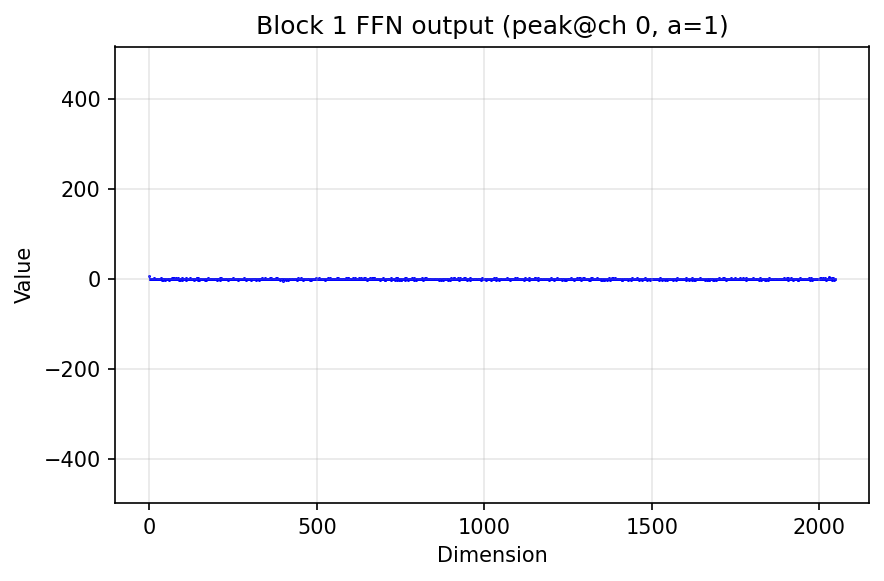}
    \end{subfigure}

    \hfill
    \begin{subfigure}[b]{0.48\textwidth}
        \centering
        \includegraphics[width=\linewidth]{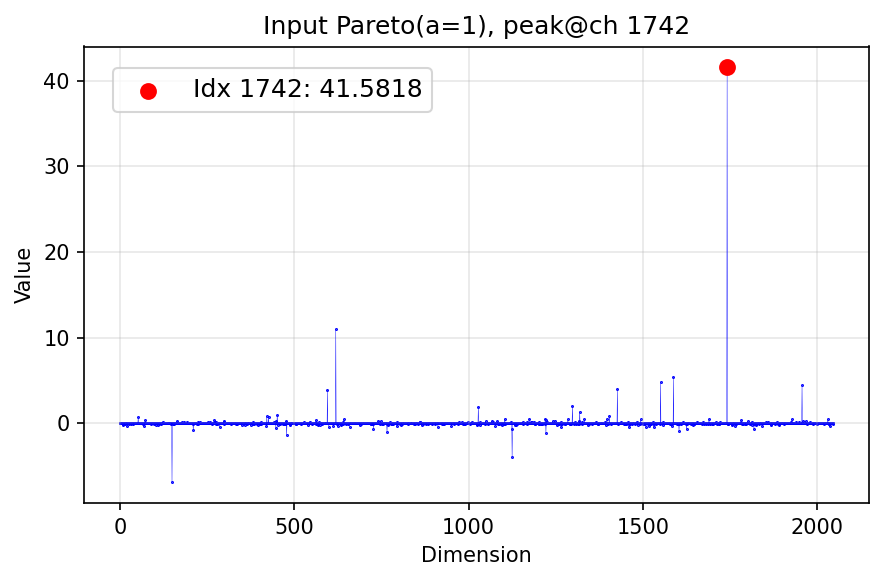}
    \end{subfigure}
    \hfill
    \begin{subfigure}[b]{0.48\textwidth}
        \centering
        \includegraphics[width=\linewidth]{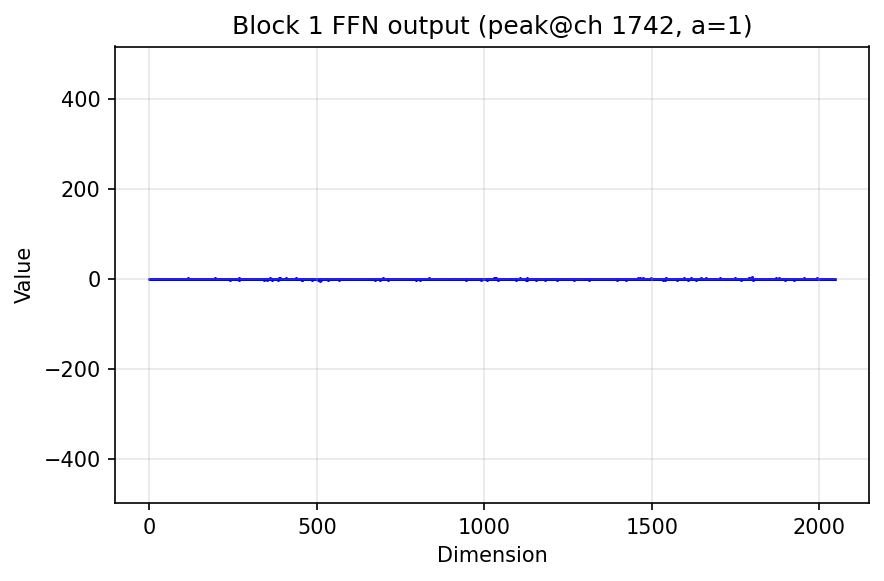}
    \end{subfigure}

    \caption{From top to bottom, the figure shows the output of the spike FFN when the input vector values are randomly sampled from a Pareto distribution ($a=1$), with the input value of the $894$th (MAGC), $0$th, or $1742$nd channel swapped with the maximum input value, respectively.}
    \label{fig:layer1FFN}
\end{figure*}

\subsection{MAGC Triggers Massive Activations}

Next, we show that this single-channel spike triggers massive activations through the computation of the spike FFN. To empirically validate this claim, we conduct the following experiment. Specifically, we randomly generate the input embedding to the spike FFN from a Pareto distribution with random signs: each channel value of the input embedding is i.i.d. drawn as $X = \text{random}(\pm) Z$ with $Z\sim\text{Pareto}(a), Z\geq 0$, where each element takes $+Z$ and $-Z$ with equal probability 0.5, and $a$ is a fixed shape parameter. Next, we swap the input value of the 894th channel with the maximum input value. We then input this randomly generated embedding into the spike FFN.
Meanwhile, we conduct a controlled experiment. Specifically, we used the same randomly sampled input embedding but swapped the input value of a randomly selected channel other than the 894th, such as the 0th or 1742nd, with the maximum value.
We report the results for $a=1$ in this section and present the results for $a=3$ and $a=5$ in Appendix \ref{appendix:a3a5}.

Figure \ref{fig:layer1FFN} depicts the randomly generated inputs and their corresponding outputs after the spike FFN. It can be observed that when swapping the value of the 894th channel with the maximum value, the outputs of the spike FFN demonstrate massive activations. Since the randomly generated inputs carry no semantic meaning, this phenomenon mainly indicates that the FFN learns to recognize MAGC.
However, swapping the value of any other channel, such the 0th or 1742nd channel, with the maximum value does not lead to the output of massive activations. 
We therefore conclude that MAGC controls the emergence of massive activations.

Moreover, as shown in Figure \ref{fig:layer1FFN_a3} and \ref{fig:layer1FFN_a5} in Appendix \ref{appendix:a3a5}, decreasing $a$ increases the maximum value generated by the Pareto distribution. 
More importantly, assigning a larger maximum value to the 894th channel of the spike FFN input results in a correspondingly larger activation magnitude in its output. Specifically, as $a$ decreases from 5 to 1, the largest output activation magnitude increases from approximately 40 to 400. These results demonstrate that the spike FFN is highly sensitive to the value of the 894th input channel, with a larger value in this channel inducing more ``massive'' activations.

\begin{keytakeaway}[Key Takeaway 1:]
~There exists a massive activation gating channel (MAGC) in the hidden representation of a spike token, whose position is fixed for a specific LLM. When the hidden representation is used as the input to the spike FFN, the FFN output exhibits massive activations when the value of the MAGC in the hidden representation is sufficiently large (or small, depending on the LLM).
\end{keytakeaway}

\section{From MAGC to Massive Activations}

In this section, we examine how a large value of MAGC induces massive activations in LLMs. Since a large value of MAGC corresponds to a large magnitude in a certain channel of the input to the FFN, we begin by analyzing the asymptotic property of the FFN with respect to any input channel. Recall the FFN computation in Equation \ref{eq:ffn}: $\boldsymbol{y}(\boldsymbol{h})=\boldsymbol{W}_{\mathrm{down}}
    \left(
    \mathrm{SiLU}(\boldsymbol{W}_{\mathrm{gate}}\boldsymbol{h})
    \odot
    \boldsymbol{W}_{\mathrm{up}}\boldsymbol{h}
    \right).$
We show that, when a fixed channel has a large magnitude (in either direction), the mechanism of the FFN can be simplified as mixing columns of $\boldsymbol{W}_{\text{down}}$. Note that the mixing mechanism depends on the direction of the channel value ($x\to +\infty$ or $x\to -\infty$), which explains why MAGC also has a fixed gating direction.

\subsection{Asymptotic Property of an FFN w.r.t. an Input Channel}

Figure \ref{fig:llama_token0_input} demonstrates the phenomenon that the input to the FFN contains a large magnitude in a certain channel for spike tokens. We begin by analyzing the asymptotic property of any FFN with respect to any input channel. Consider a gated FFN in Equation \ref{eq:ffn}, for input $\boldsymbol{h}\in\mathbb{R}^d$,
\begin{align}
    \text{FFN}(\boldsymbol{h}) = \boldsymbol{y}(\boldsymbol{h})
    =
    \boldsymbol{W}_{\mathrm{down}}
    \left(
    \mathrm{SiLU}(\boldsymbol{W}_{\mathrm{gate}}\boldsymbol{h})
    \odot
    \boldsymbol{W}_{\mathrm{up}}\boldsymbol{h}
    \right).
\end{align}
where $\boldsymbol{W}_{\mathrm{up}} \in \mathbb{R}^{m\times d}$, $\boldsymbol{W}_{\mathrm{gate}} \in \mathbb{R}^{m\times d}$, and $\boldsymbol{W}_{\mathrm{down}} \in \mathbb{R}^{d\times m}$ are parameter matrices, $\mathrm{SiLU}(x):=x \cdot \sigma(x)$ with the logistic sigmoid function $\sigma(x):=\frac{1}{1+\exp(-x)}$, and $\odot$ is the element-wise product.

We focus on the input $\boldsymbol{h}$ controlled by a fixed single channel $i$. For any value $x\in \mathbb{R}$ in the fixed channel $i$, the input can be written as,
\begin{align}
    \boldsymbol{h}(x)=x \cdot \boldsymbol{e}_{i}+ \boldsymbol{\bar h}. \label{equation:hx}
\end{align}
Here, \(\boldsymbol{\bar h} \in\mathbb{R}^{d}\) is the representation outside the focused channel \(i\) with $\boldsymbol{\bar h}(i)=0$, and $\boldsymbol{e}_{i}$ is the unit vector at position $i$. Next, we analyze the asymptotic property of $\boldsymbol{y}(\boldsymbol{h}(x))$ for $x\to +\infty$ without loss of generality.
Let $\boldsymbol{\omega}_j={\boldsymbol{W}_{\mathrm{down}}^{(:,j)}} \big/ {\left\|\boldsymbol{W}_{\mathrm{down}}^{(:,j)}\right\|_2}$ be the normalized $j$-th column of $\boldsymbol{W}_{\mathrm{down}}$.
The following theorem shows the closed form of $\boldsymbol{y}(\boldsymbol{h}(x))$ when $x\to +\infty$.

\begin{theorem}
\label{theorem:cs1}
Given input $\boldsymbol{h}(x)=x \cdot \boldsymbol{e}_{i}+ \boldsymbol{\bar h}$, when $x\to +\infty$, the output of an FFN has a quadratic form $\boldsymbol{y}(\boldsymbol{h}(x)) = x^2 \cdot \boldsymbol{y}^*+o(x^2)$ with constant $\boldsymbol{y}^*\in \mathbb{R}^d$ equal to
\begin{align}
    \boldsymbol{y}^* = \sum_{j=1}^m \alpha_j\cdot \boldsymbol{\omega}_j,
    \label{eq:y_star}
\end{align}
where 
$\boldsymbol{\alpha} = (\alpha_1,\dots,\alpha_m)$ is the combination of coefficients in the weight matrices:
\begin{align}
    \alpha_j = \begin{cases}
        W_{\mathrm{gate}}^{(j,i)} W_{\mathrm{up}}^{(j,i)} \left\| \boldsymbol{W}_{\mathrm{down}}^{(:,j)} \right\|_2 & W_{\mathrm{gate}}^{(j,i)} > 0 \\
        0 & W_{\mathrm{gate}}^{(j,i)} \leq  0
    \end{cases}.
\end{align}
Proof. See Appendix \ref{appendix:proof1}.
\end{theorem}

Specifically, for any input channel $i$ in $\boldsymbol{h}$, $\mathrm{SiLU}(\boldsymbol{W}_{\mathrm{gate}}\boldsymbol{h})$ acts as a gating function, choosing column $j$ only when $W_{\mathrm{gate}}^{(j,i)} x >0$, and $\alpha_j$ controls the sensitivity to each $\boldsymbol{\omega}_j$. From Theorem~\ref{theorem:cs1}, when \(x\) is large enough, the output $\boldsymbol{y}(\boldsymbol{h}(x)) = x^2 \cdot \boldsymbol{y}^*+o(x^2)$ exhibits a pattern of massive activations only if $\boldsymbol{y}^*$ has the same pattern. 
Similarly, an analogous result holds as $x \to -\infty$. 


\subsection{From MAGC to Massive Activations}

We introduce several metrics below to quantify the degree of massive activations.
Since massive activations describe the phenomenon in which certain activations exhibit large magnitudes, we utilize the energy distribution to characterize their magnitude distribution.
\begin{definition}[Energy Distribution] \label{def:energy}
    Given vector $\boldsymbol{v}=(v_1,\dots,v_d)\in \mathbb{R}^d$, let its energy distribution be $\boldsymbol{p}(\boldsymbol{v})=(p_1(\boldsymbol{v}),\dots,p_d(\boldsymbol{v}))$, where
    \begin{align}
        p_i(\boldsymbol{v})=\frac{v^2_i}{\|\boldsymbol{v}\|_2^2},\quad \forall i=1,\dots,d.
    \end{align}
\end{definition}
The energy distribution is a probability distribution. Moreover, if $\boldsymbol{y}$ exhibits massive activations, $\boldsymbol{p}(\boldsymbol{y})$ should be highly concentrated. We define the degree of concentration as the negative entropy of $\boldsymbol{p}(\boldsymbol{y})$ plus a constant. Formally, 
\begin{definition}[Concentration Degree] \label{def:concentrate}
    Given vector $\boldsymbol{v}\in \mathbb{R}^d$, let its concentration degree be
    \begin{align}
        R(\boldsymbol{v})=\log d + \sum_{i=1}^d p_i(\boldsymbol{v})\log p_i(\boldsymbol{v})
    \end{align}
    where $\boldsymbol{p}(\boldsymbol{v})=(p_1(\boldsymbol{v}),\dots,p_d(\boldsymbol{v}))$ is the energy distribution defined in Definition~\ref{def:energy}.
\end{definition}
The concentration degree $R(\boldsymbol{v})\in [0,\log d]$ is exactly the negative entropy plus a constant $\log d$. A larger $R(\boldsymbol{v})$ means $\boldsymbol{v}$ is more concentrated. 
$R(\boldsymbol{v})=0$ means $\boldsymbol{v}$ is uniform, while $R(\boldsymbol{v})=\log d$ means $\boldsymbol{v}$ has only one non-zero entry.

We now fix MAGC as the focused channel $i^*$ and analyze the relationship between the concentration degree of the output $\boldsymbol{y}^*$ and those of $\boldsymbol{\omega}_j$ and $\boldsymbol{\alpha}$. We first find that the following assumption is empirically satisfied.

\begin{assumption}
 When MAGC is the focused channel, $\boldsymbol{\alpha}$ has a high concentration degree, i.e., $R(\boldsymbol{\alpha}) \geq \log m - \epsilon$ for a small $\epsilon$. Let $j^*=\arg\max_j p_j(\boldsymbol{\alpha})$ be the dominant column index; then
 \begin{align}
    R(\boldsymbol{\alpha}) \geq \log m - \epsilon \Rightarrow p_{j^*}(\boldsymbol{\alpha}) \geq e^{-\epsilon} .
\end{align}
\end{assumption}

\textbf{Empirical verification.} Rather than assuming this property, we verify it directly by scanning all $d=2{,}048$ input channels of the spike FFN in LLaMA-3.2-1B: for each channel $i$ we form the coefficient vector $\boldsymbol{\alpha}$ from Theorem~\ref{theorem:cs1} and measure its concentration $R(\boldsymbol{\alpha})$, dominant energy share $p_{j^*}(\boldsymbol{\alpha})$, and norm $\|\boldsymbol{\alpha}\|_2$. MAGC ($i^*=894$) attains the largest concentration among all channels, with $R(\boldsymbol{\alpha})=9.009$ against $\log m=9.011$ (i.e.\ $\epsilon\approx 1.9\times 10^{-3}$) and $p_{j^*}(\boldsymbol{\alpha})=0.9998$, whereas the mean over all channels is only $R(\boldsymbol{\alpha})=2.908$.
This confirms that MAGC is a significant outlier in $\boldsymbol{\alpha}$-concentration. 

\begin{figure*}[htbp]
    \centering

    \begin{subfigure}[b]{0.48\textwidth}
        \centering
        \includegraphics[width=\linewidth]{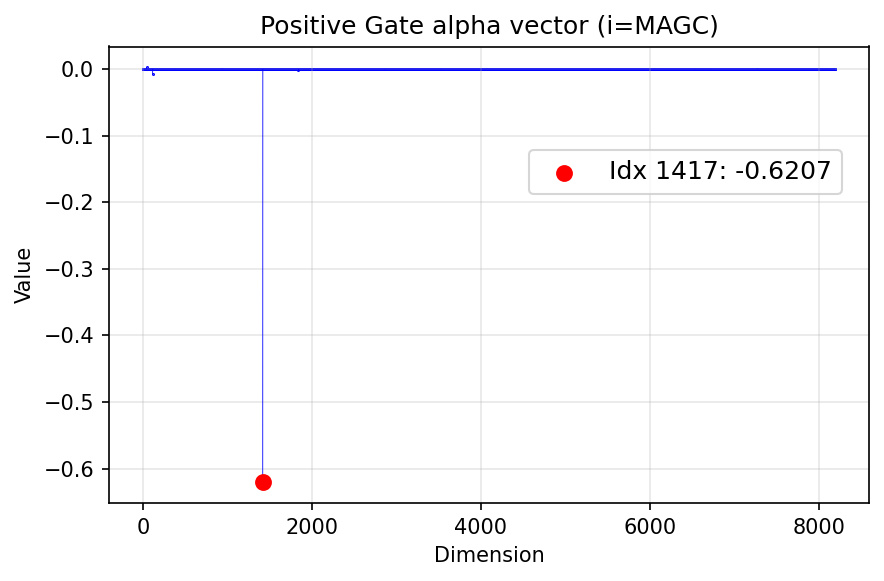}
        \caption{$\boldsymbol{\alpha}$ vector when $i^*=894$ for LLaMA-3.2-1B.}
        \label{fig:llama_alpha}
    \end{subfigure}
    \hfill
    \begin{subfigure}[b]{0.48\textwidth}
        \centering
        \includegraphics[width=\linewidth]{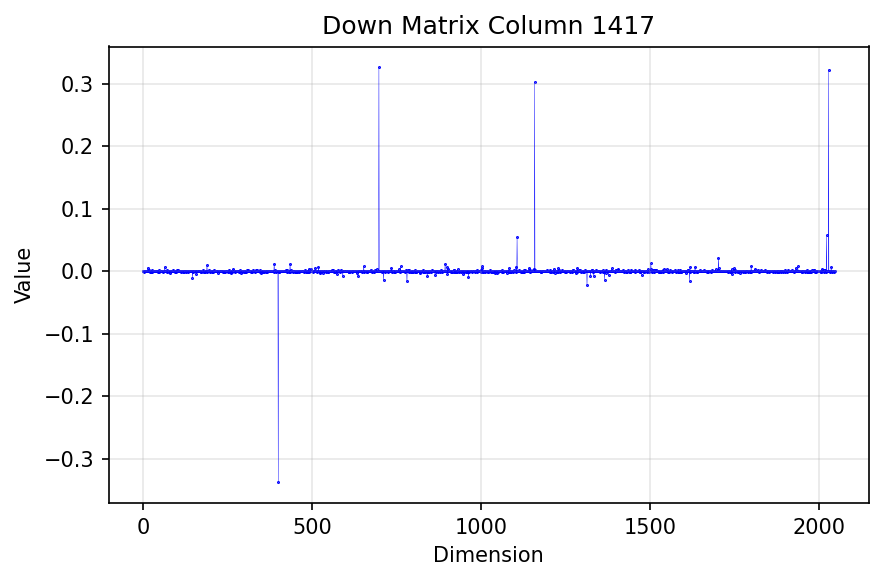}
        \caption{$\boldsymbol{\omega}_{1417}$ for LLaMA-3.2-1B.}
        \label{fig:llama_down_column}
    \end{subfigure}

    \hfill
    
    \begin{subfigure}[b]{0.48\textwidth}
        \centering
        \includegraphics[width=\linewidth]{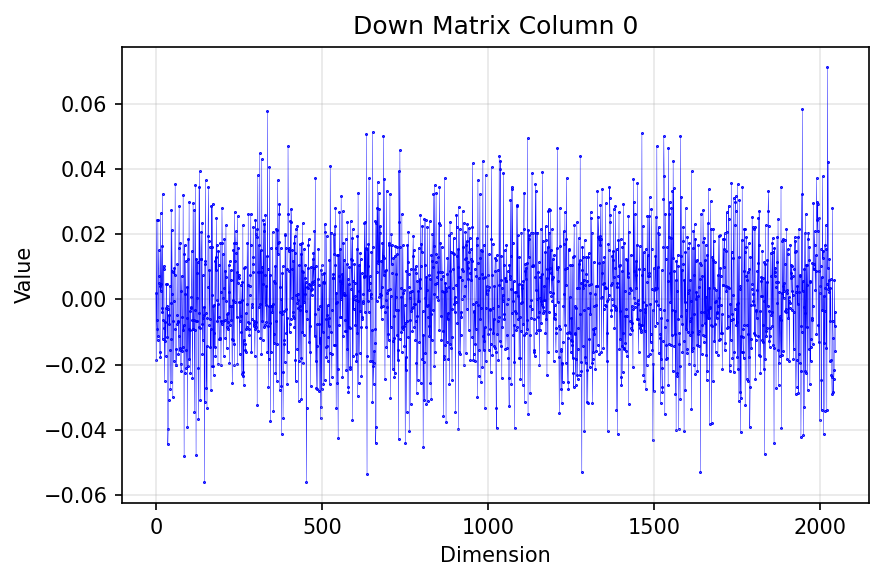}        \caption{$\boldsymbol{\omega}_{0}$ for LLaMA-3.2-1B.}
        \label{fig:llama_down_column_0}
    \end{subfigure}
    \hfill
    \begin{subfigure}[b]{0.48\textwidth}
        \centering
        \includegraphics[width=\linewidth]{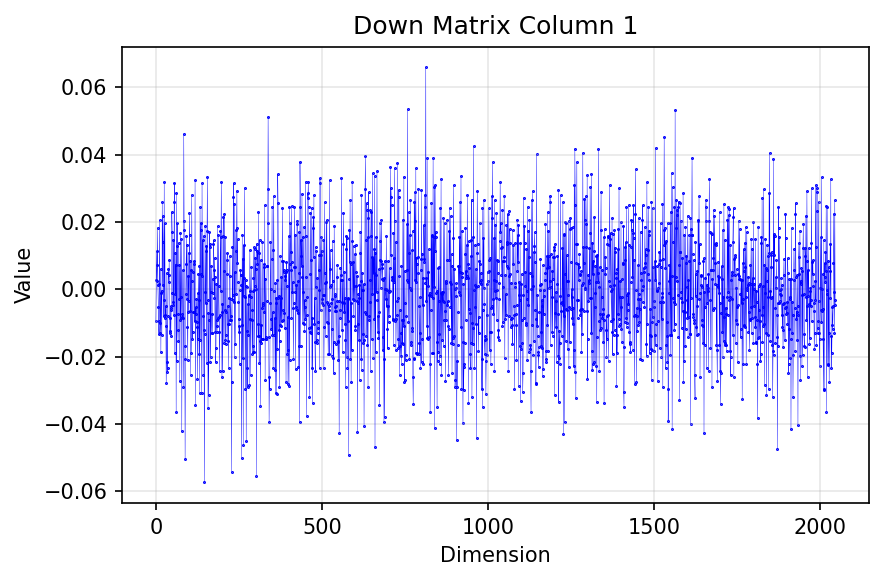}
        \caption{$\boldsymbol{\omega}_{1}$ for LLaMA-3.2-1B.}
        \label{fig:llama_down_column_1}
    \end{subfigure}

    \caption{Visualization results for LLaMA-3.2-1B.}
    \label{fig:llama_combined}
\end{figure*}

A high concentration degree $R(\boldsymbol{\alpha})$ indicates that a few $p_j(\boldsymbol{\alpha})$ are extremely large. We empirically demonstrate $\boldsymbol{\alpha}$ when $i^*=894$ for LLaMA-3.2-1B in Figure \ref{fig:llama_alpha} and highlight the largest magnitude in the vector. As shown, $\boldsymbol{\alpha}$ is highly concentrated with $j^*=1417$. That is, the output $\boldsymbol{y}$ of the FFN is largely controlled by the $j^*$-th column of $\boldsymbol{W}_{\mathrm{down}}$.
Theoretically, based on the high concentration of $\boldsymbol{\alpha}$, we next quantify the contribution of the remaining columns.
\begin{definition}
\label{def:kappa}
 Given $j^*=\arg\max_j p_j(\boldsymbol{\alpha})$, define the residual amplification factor
\[ \kappa=\frac{\left\|\sum_{j\neq j^*}\alpha_j \boldsymbol{\omega}_j\right\|_2}{\sqrt{\sum_{j\neq j^*}\alpha_j^2 }}.
\]
\end{definition}
The residual amplification factor \(\kappa\) quantifies the contribution of other channels after removing the dominant column $j^*$. 
In the following proposition, \(\kappa\) controls how strongly the non-dominant columns can perturb the energy distribution of the output. 

\begin{proposition} \label{prop:other_channel}
Let 
\begin{align}
   \tau=\kappa \sqrt{\frac{1-p_{j^*}(\boldsymbol{\alpha})}{p_{j^*}(\boldsymbol{\alpha})}}.
\end{align}
We have
\begin{align}
    \left|R(\boldsymbol{y}^*)-R(\boldsymbol{\omega}_{j^*})\right|
\leq
\Delta\left(|\sin\angle(\boldsymbol{y}^*,\boldsymbol{\omega}_{j^*})| \right)\leq
\Delta\left(
\tau
\right),
\end{align}
where for $t \in\left(0,1-\frac{1}{d}\right)$, the function $\Delta$ is defined as
$
\Delta(t)
=
-t\log t-(1-t)\log(1-t)+t\log(d-1).
$

Proof. See Appendix~\ref{appendix:proof2}. 
\end{proposition}


Proposition \ref{prop:other_channel} indicates that when $p_{j^*}(\boldsymbol{\alpha})$ is large (i.e., close to 1), we have $\tau\approx 0$, and there is almost no gap between the concentration degree of $\boldsymbol{y}^*$ and that of $\boldsymbol{\omega}_{j^*}$.

In the case of LLaMA-3.2-1B, the output $\boldsymbol{y}$ of the FFN is largely controlled by the $j^*$-th (1417th) column of $\boldsymbol{W}_{\mathrm{down}}$. We next depict $\boldsymbol{\omega}_{j^*}$ in Figure \ref{fig:llama_down_column}. Comparing Figure \ref{fig:llama_down_column} to \ref{fig:llama_token0_output}, we find that these two vectors are highly co-linear, empirically with $|\sin\angle(\boldsymbol{y}^*,\boldsymbol{\omega}_{j^*})| = 0.0055$. For other columns such as columns 0 and 1 in Figures \ref{fig:llama_down_column_0} and \ref{fig:llama_down_column_1}, the parameters have lower maximum magnitudes and do not contain the pattern of massive activations. Note that both $\boldsymbol{y}^*$ and $\boldsymbol{\omega}_{j^*}$ are computed only from the pretrained weight matrices. We show that when the value of channel $i^*$ is large enough, the actual output $\boldsymbol{y}$ is dominated by $\boldsymbol{\omega}_{j^*}$, which leads to $\boldsymbol{y} \approx x^2 \boldsymbol{y}^* \approx x^2 \alpha_{j^*} \boldsymbol{\omega}_{j^*}$. If $\alpha_{j^*}$ has a large magnitude and $\boldsymbol{\omega}_{j^*}$ is highly concentrated (i.e., contains the pattern of massive activations), the output $\boldsymbol{y}$ should exhibit the same pattern. Below, we summarize the mechanism by which the MAGC in the input to a spike FFN leads to its output that exhibits massive activations.

\begin{keytakeaway}[Key Takeaway 2:]
~When the MAGC of an input embedding to the spike FFN takes a sufficiently large (or small) value, the output of the FFN asymptotically approaches $x^2 \sum_{j=1}^m \alpha_j \boldsymbol{\omega}_j$, where $\boldsymbol{\omega}_j$ is the normalized $j$-th column of $\boldsymbol{W}_{\text{down}}$ of the FFN and $x$ is the value of the MAGC. Moreover, the coefficients $\boldsymbol{\alpha}=(\alpha_1,\dots,\alpha_m)$ concentrate on one (or a few) dimensions. As a result, the FFN output is reduced to one normalized column of $\boldsymbol{W}_{\text{down}}$ or a combination of a few normalized columns of $\boldsymbol{W}_{\text{down}}$, multiplied by their respective coefficient and $x^2$. Since these columns exhibit the shape of massive activations, the output therefore exhibits massive activations.
\end{keytakeaway}

\section{Extension}
\label{sec:extension}

\subsection{Extension to Other LLMs}
\label{sec:llm_extension}

\begin{table}[htbp]
    \centering
    \setlength{\tabcolsep}{5pt}
    \small
    \begin{tabular}{l|cccccc}
    \toprule
        Model Name & \makecell{Spike FFN Block} & MAGC($i^*$) & Direction & Concentrated Columns($J$)\\
        \midrule
        LLaMA-3.2-1B & 1 & 894 & + & 1417 \\
        LLaMA-2-7B & 1 & 310 & + & 10411, 7890 \\ 
        Qwen2.5-1.5B & 1 & 1395 & - & 5601, 3295, 3750 \\ 
        QwQ-32B & 4 & 4626  & - & 11696, 1368 \\
        Ministral-3-3B & 2 & 2  & - & 0, 2\\ 
        Deepseek-R1-8B & 1 & 717  & + & 2427, 198 \\ 
        \bottomrule
    \end{tabular}
    \caption{MAGC in different models. The spike FFN block, MAGC $i^*$, and concentrated column ids are counted from 0. 
    The concentrated columns are identified by first selecting the three columns ${\omega}_j$ of the down-projection matrix with the largest absolute $\alpha_j$ values and then filtering out the selected columns that do not satisfy $|\cos(\boldsymbol{\omega}_j,\boldsymbol{y}_0)| > 0.9$.
    }
    \label{tab:model_sum}
\end{table}

In this section, we examine the generality of our findings across LLMs of different model families and sizes. Specifically, we study LLaMA-2-7B, Qwen2.5-1.5B, QwQ-32B, Ministral-3-3B-Instruct-2512, and DeepSeek-R1-Distill-Llama-8B. Together with LLaMA-3.2-1B, these models cover four model families and range from 1B to 32B parameters. 
Based on our empirical findings, we propose an empirical method to identify the MAGC as detailed in Appendix \ref{appendix:results}. 
For each model, we repeat the procedure in the preceding sections: we localize the spike FFN, identify MAGC $i^*$ and its gating direction, and find the dominant column(s) $J$ of $\boldsymbol{W}_{\text{down}}$. Table \ref{tab:model_sum} summarizes the results, and the corresponding visualizations are provided in Appendix \ref{appendix:results}. 

We highlight the following findings. 
First, the spike FFN consistently occurs near the beginning of the network. Notably, even for QwQ-32B with 64 transformer blocks, the spike FFN is located at block 4. 
Second, the position $i^*$ of MAGC is fixed for a specific model but differs across models, as does the gating direction. 
Third, the dominant column is not always unique. For LLaMA-3.2-1B, the energy distribution $p_j(\boldsymbol{\alpha})$ is dominated by a single column of $\boldsymbol{W}_{\text{down}}$. For the remaining models, however, two or three columns share the dominant energy. 
As shown in Appendix \ref{appendix:results}, the columns in $J$ exhibit highly similar shapes and concentrate their mass on the same small set of output channels, so their weighted combination $\boldsymbol{y}^*$ remains concentrated and reproduces the pattern of massive activations. In contrast, a randomly selected column distributes its mass much more broadly across channels.




\subsection{Relation to Attention Sink}
\label{sec:app_sink}


\begin{wrapfigure}{r}{0.5\linewidth}
    \centering
    \vspace{-\intextsep}
    \includegraphics[width=\linewidth]{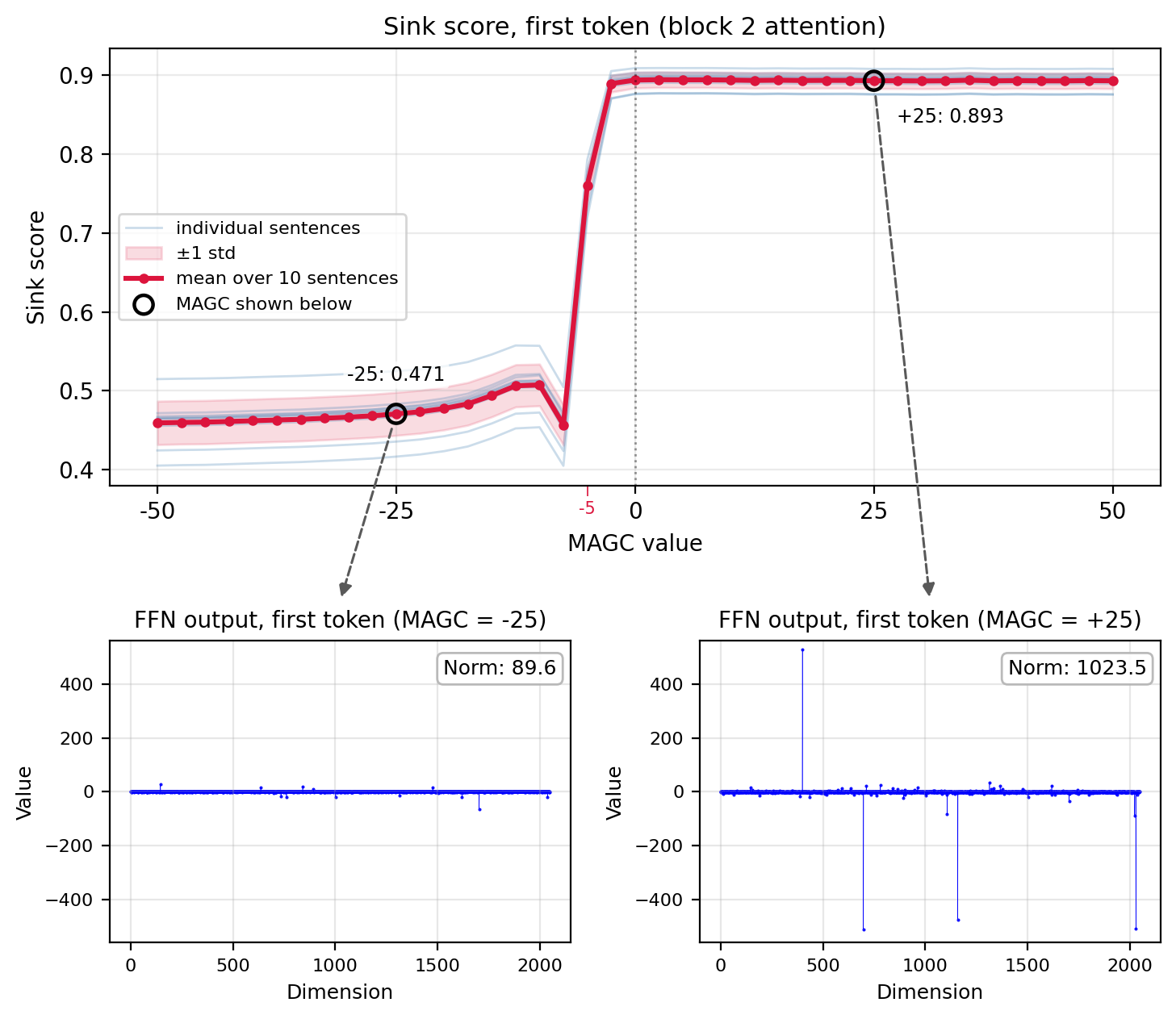}
    \caption{Sink score of the first token in LLaMA-3.2-1B, measured in the attention of block 2, as a function of the value assigned to MAGC ($i^*=894$) of the first token at the input of the spike FFN (block 1). Thin curves correspond to the ten individual sentences, the thick curve to their mean, and the shaded area to one standard deviation. The unmodified average value of MAGC is $8.81$.}
    \label{fig:magc_sink}
\end{wrapfigure}
Attention sink is commonly attributed to massive activations \citep{sun2024massive,yan2024unveiling}. Our previous results show that massive activations can be controlled by MAGC. We now extend this finding to adjust attention sink using MAGC.

\textbf{Experimental Setup.} We feed ten natural sentences to LLaMA-3.2-1B and intervene only on MAGC of the first token: at the input to the spike FFN (block 1), we overwrite $h_0(i^*)$ with a prescribed value $c$ and leave all other channels and all other tokens unchanged. We then measure the attention sink of the first token in the following block (block 2) using the sink score \citep{gu2025attention},
\begin{equation}
\label{eq:sink_score}
    s_0 = \frac{1}{H(T+1)}\sum_{h=0}^{H-1}\sum_{t=0}^{T} A_{t,0}^{h},
\end{equation}
where $A_{t,0}^{h}$ denotes the post-softmax attention that the $t$-th query assigns to the first token in the $h$-th head, $T+1$ is the sequence length, and $H$ is the number of attention heads. Since the first token is visible to all queries under the causal mask, $s_0$ is the average fraction of attention that the first token receives, and $s_0$ close to 1 indicates a dominant attention sink.

\textbf{Results.} 
Figure \ref{fig:magc_sink} reports the sink score $s_0$ as a function of MAGC value $c$. Without intervention, the original MAGC value is $8.81$, at which the first token attracts $s_0=0.89$ of the attention mass. 
As $c$ varies, the sink score shows a sharp step-like transition. 
For $c\geq -2.5$, $s_0$ stays near $0.89$, together with massive activations, as shown at $c=+25$. 
For sufficiently negative $c$, the sink collapses. The sink score drops to about $0.46$ for $c\leq -7.5$, while the massive activations disappear, as shown at $c=-25$. The transition is consistent across sentences, with standard deviation below $0.03$ throughout the sweep. Therefore, adjusting the MAGC value in the embedding of an input spike token to the spike FFN can adjust the sink mass on the token.




\section{Conclusion and Future Work }

In this study, we show how tokens develop massive activations in pretrained transformers and identify the massive activation gating channel (MAGC): a channel in the input to an early FFN whose value gates the emergence of massive activations. We show both empirically and theoretically that once MAGC is large or small enough, the spike FFN asymptotically collapses into a quadratic mixing of a few highly concentrated columns of the down-projection matrix, which reproduces the massive-activation pattern. This finding holds across six LLMs from four families ranging from 1B to 32B parameters.

Our analysis has several limitations that suggest directions for future work. First, it focuses on decoder-only LLMs with gated (such as SwiGLU) FFNs; extending the framework to mixture-of-experts architectures is a natural next step. Second, we localize a single spike FFN and a single gating channel per model, whereas a few models exhibit more diffuse behavior across several channels or columns; a multi-channel account may be needed for such cases. Finally, our study is purely mechanistic at inference time and does not explain why a particular channel becomes MAGC during pretraining, which we leave for future investigation.

\subsection*{AI use statement}

In this work, we used generative AI tools for paper proofreading and code writing. We have not used generative AI tools for literature review, idea formalization, and drafting the manuscript. We have reviewed all AI-assisted work. For example, LLM-generated code was verified by the authors. We take responsibility for the originality of this work.

\subsection*{Reproducibility statement}

We have released all code needed to reproduce the results in this work at GitHub\footnote{\url{https://github.com/MinjiaMao/magc}}. It includes scripts for identifying the MAGC, running the intervention experiments, and generating figures. All experiments use publicly available pretrained models. 

\bibliography{iclr2025_conference}
\bibliographystyle{iclr2025_conference}

\newpage 
\appendix

\setcounter{figure}{0}
\setcounter{table}{0}
\setcounter{section}{0}
\setcounter{equation}{0}

\renewcommand{\thetable}{A\arabic{table}}
\renewcommand{\thefigure}{A\arabic{figure}}
\renewcommand{\theequation}{A\arabic{equation}}

\section{Generalization of Near Identity Assumption for SiLU}
\label{appendix:violation}

Following \citet{sun2026spike}, we compute the cosine similarity between the gate pre-activation $\boldsymbol{W}_{\text{gate}}\boldsymbol{h}_t$ and its activation $\text{SiLU}(\boldsymbol{W}_{\text{gate}}\boldsymbol{h}_t)$, and the ratio of their $\ell_2$ norms for spike tokens.
Using the randomly generated sequences described in Section~3, we report in Table~\ref{tab:SwiGLU_identity} both metrics averaged over the first tokens of these sequences. For SiLU to behave approximately as an identity function, both metrics should be close to 1. Among the models examined, this condition is approximately satisfied only by LLaMA-2-7B.

\begin{table}[htbp]
    \centering
    \setlength{\tabcolsep}{5pt}
    \small
    \begin{tabular}{llc}
    \toprule
    Model & Metric & Token 0 \\
    \midrule
    \multirow{2}{*}{LLaMA-2-7B}
    & Cosine sim. & 0.9719   \\
    & Norm ratio  & 0.8103   \\
    \midrule
    \multirow{2}{*}{LLaMA-3.2-1B}
    & Cosine sim. & 0.8623  \\
    & Norm ratio  & 0.6232  \\
    \midrule
    \multirow{2}{*}{Qwen2.5-1.5B}
    & Cosine sim. & 0.2276   \\
    & Norm ratio  & 0.1564   \\
    \midrule
    \multirow{2}{*}{QwQ-32B}
    & Cosine sim. & 0.2990  \\
    & Norm ratio  & 0.1186  \\
    \midrule
    \multirow{2}{*}{Ministral-3-3B}
    & Cosine sim. & 0.7895  \\
    & Norm ratio  & 0.7897  \\
    \midrule
    \multirow{2}{*}{Deepseek-R1-8B}
    & Cosine sim. & 0.9215   \\
    & Norm ratio  & 0.7720   \\
    \bottomrule
    \end{tabular}
    \caption{Input-output characteristics of SiLU in the spike FFN across models. }
    \label{tab:SwiGLU_identity}
\end{table}

\begin{figure*}[htb]
    \centering
    \begin{subfigure}[b]{0.7\textwidth}
        \centering
        \includegraphics[width=\linewidth]{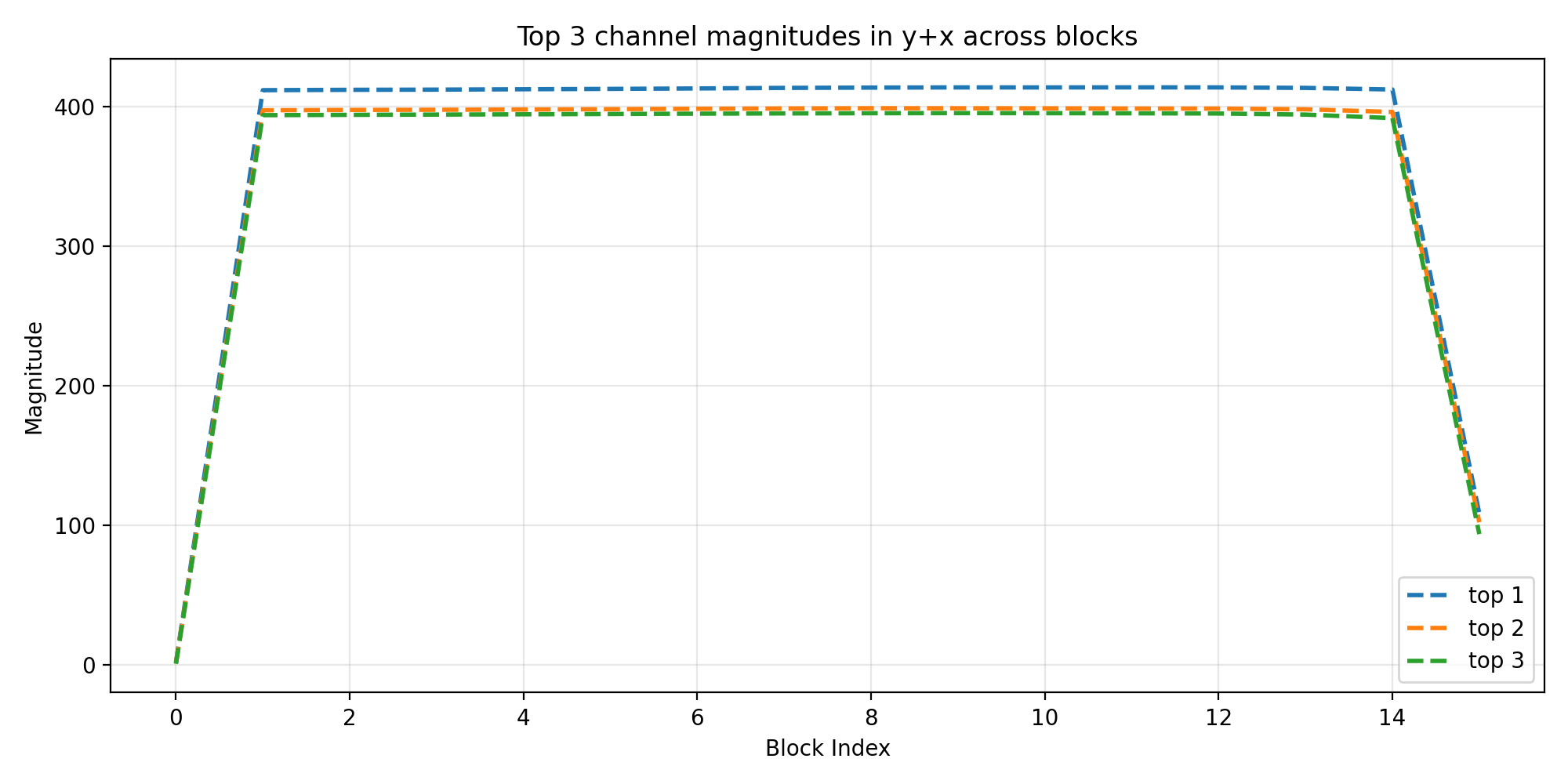}
    \end{subfigure}
    \hfill
    \begin{subfigure}[b]{0.7\textwidth}
        \centering
        \includegraphics[width=\linewidth]{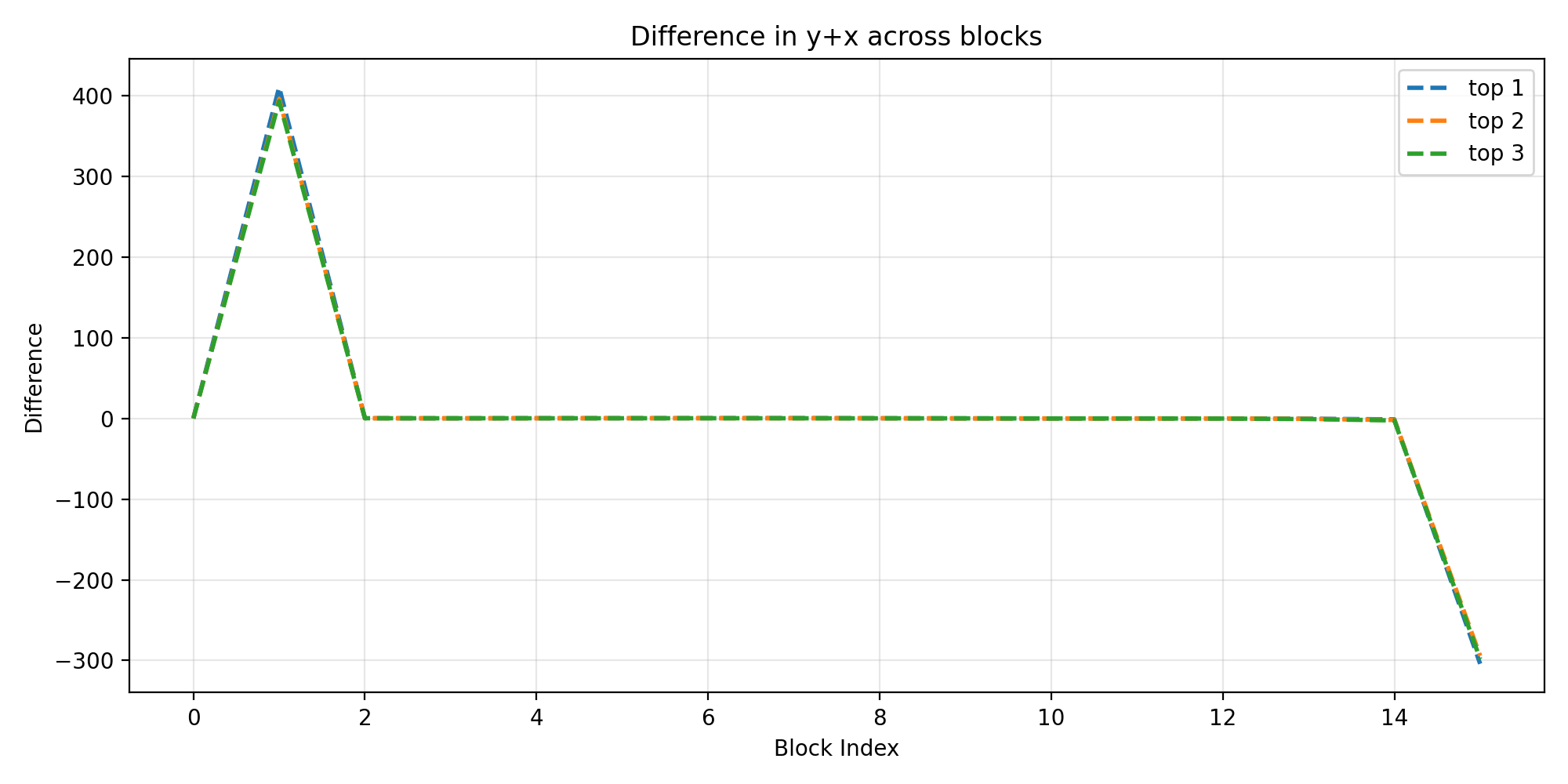}
    \end{subfigure}
    \caption{Top-3 channel magnitudes of transformer block outputs $x+u+y$ in Llama-3.2-1B and their differences.}
    \label{fig:llama1_top3}
\end{figure*}

\section{Additional Experiments for LLaMA-3.2-1B}
\label{appendix:exp_llama1}

\begin{figure*}[htbp]
    \centering

    \begin{subfigure}[b]{0.35\textwidth}
        \centering
        \includegraphics[width=\linewidth]{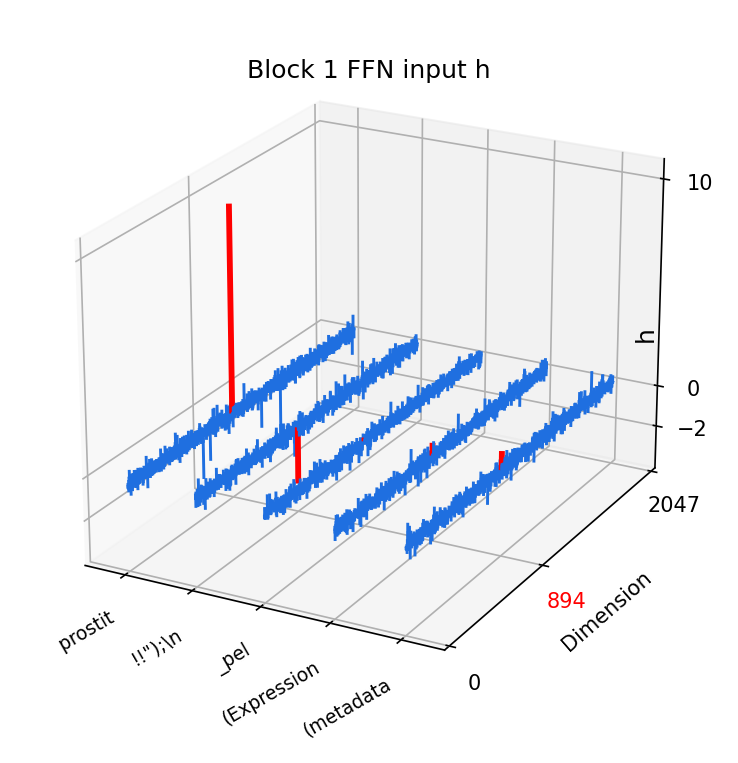}
    \end{subfigure}
    \begin{subfigure}[b]{0.35\textwidth}
        \centering
        \includegraphics[width=\linewidth]{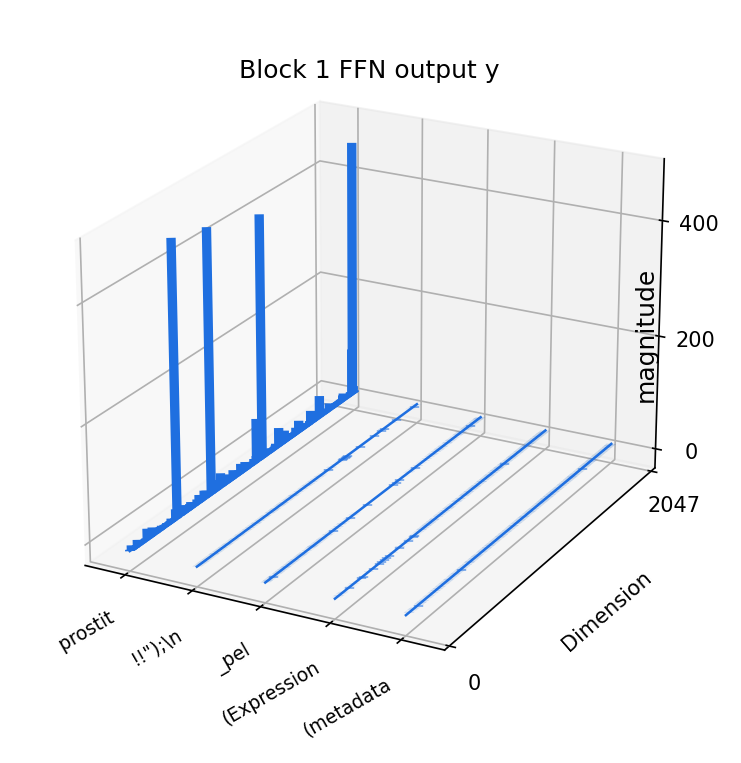}
    \end{subfigure}

    \begin{subfigure}[b]{0.35\textwidth}
        \centering
        \includegraphics[width=\linewidth]{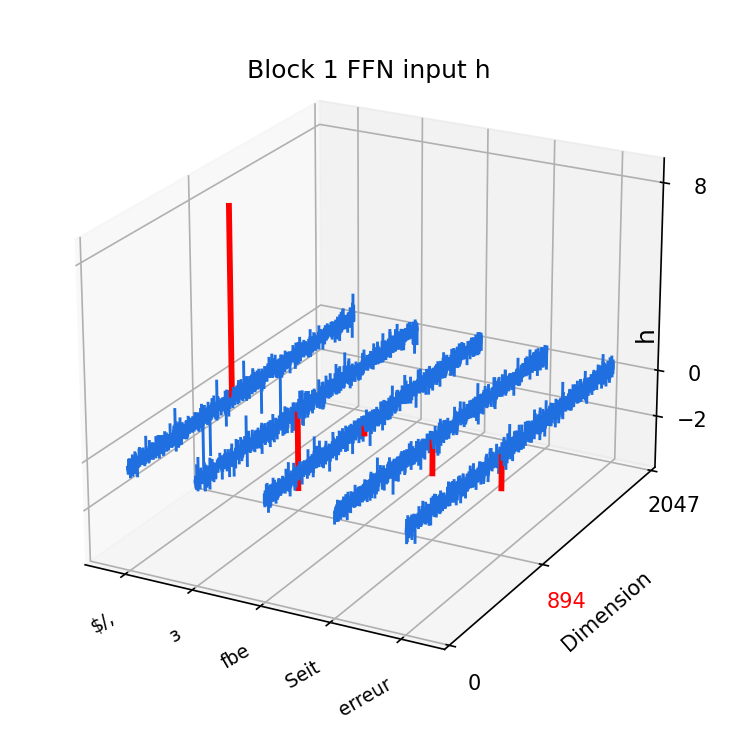}
    \end{subfigure}
    \begin{subfigure}[b]{0.35\textwidth}
        \centering
        \includegraphics[width=\linewidth]{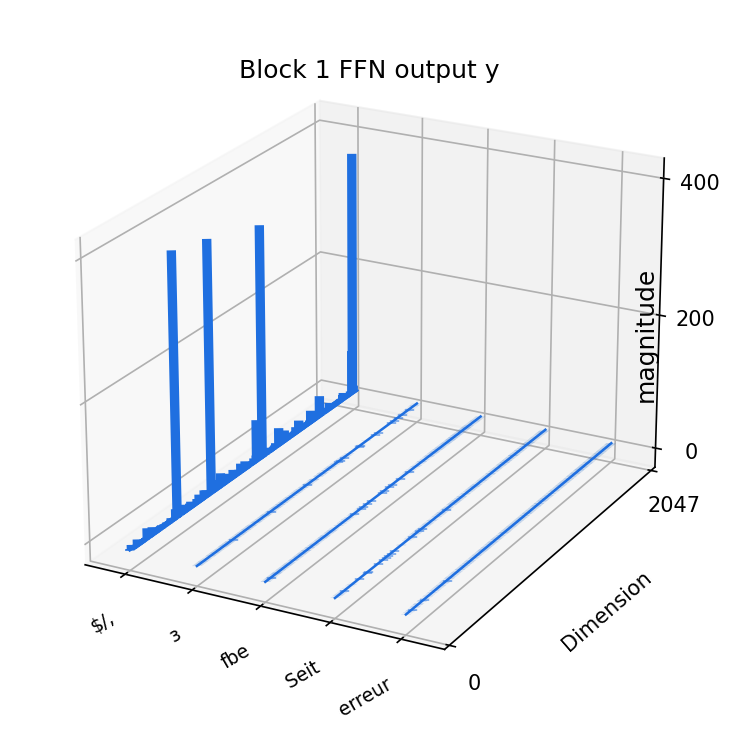}
    \end{subfigure}

    \begin{subfigure}[b]{0.35\textwidth}
        \centering
        \includegraphics[width=\linewidth]{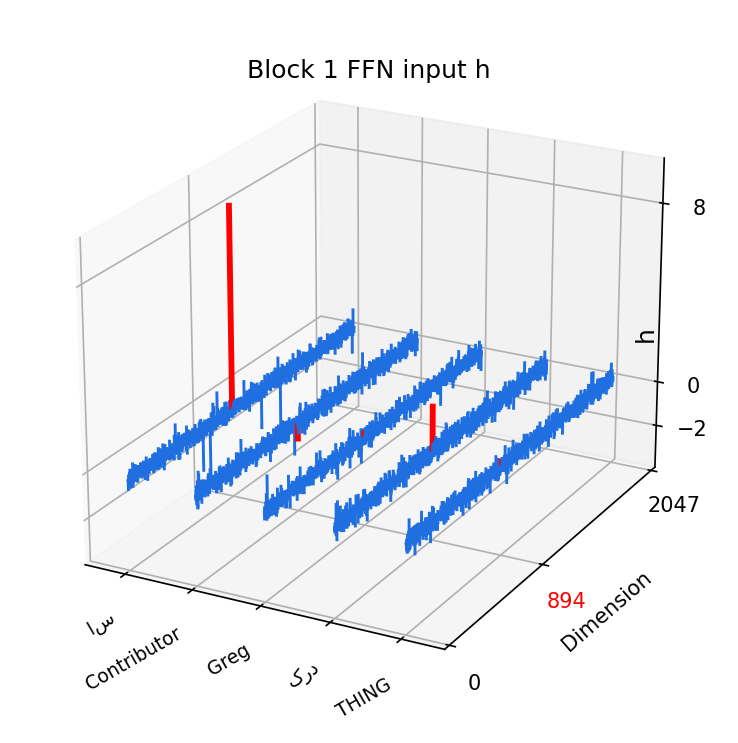}
    \end{subfigure}
    \begin{subfigure}[b]{0.35\textwidth}
        \centering
        \includegraphics[width=\linewidth]{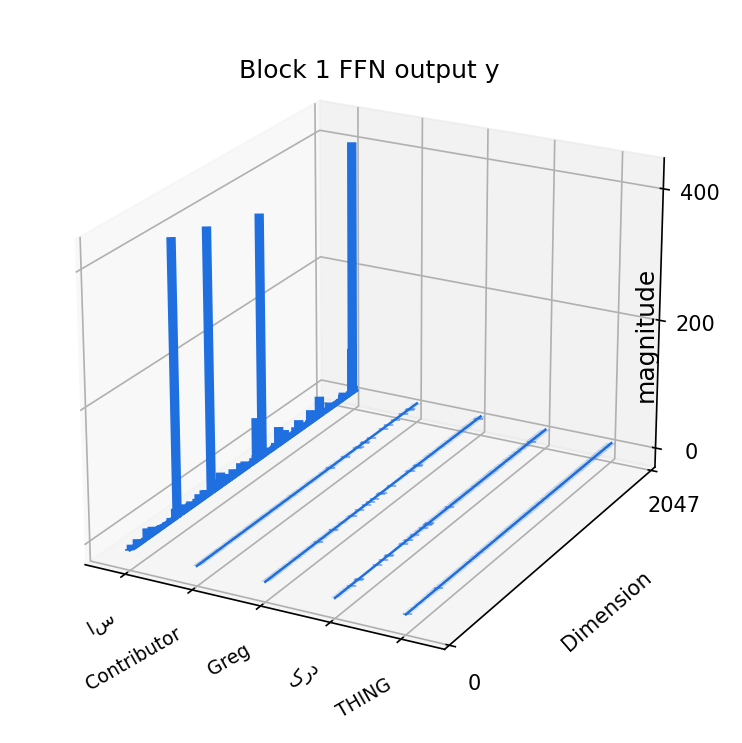}
    \end{subfigure}

    \begin{subfigure}[b]{0.35\textwidth}
        \centering
        \includegraphics[width=\linewidth]{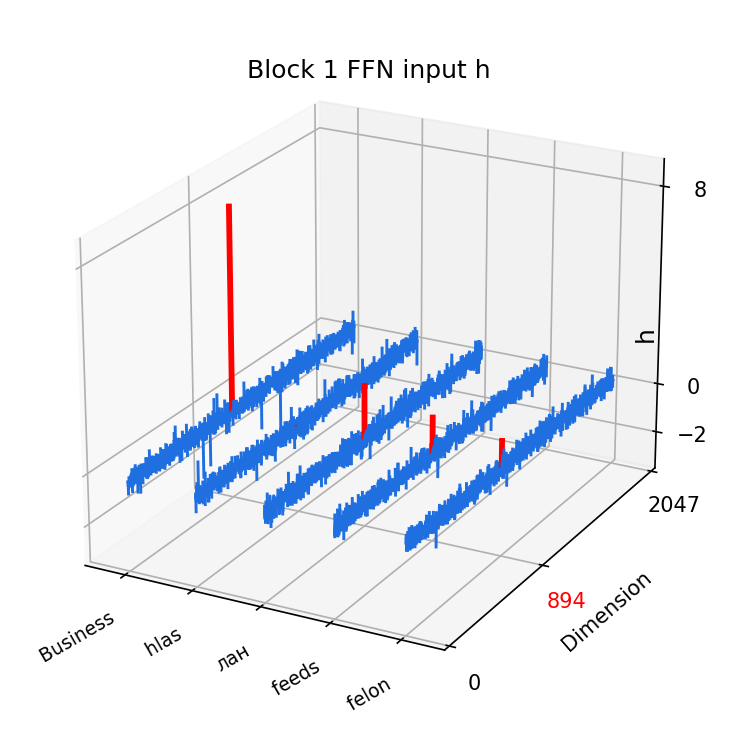}
    \end{subfigure}
    \begin{subfigure}[b]{0.35\textwidth}
        \centering
        \includegraphics[width=\linewidth]{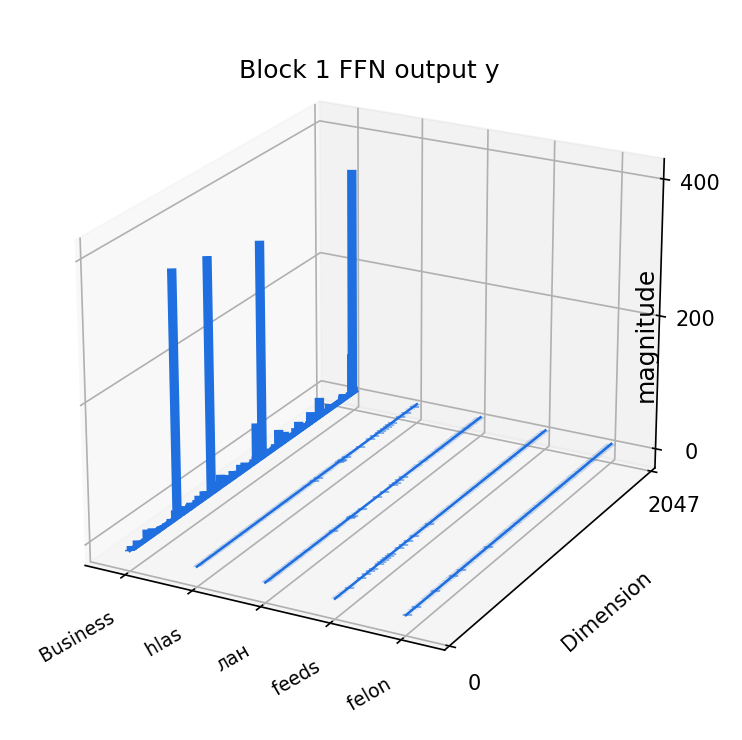}
    \end{subfigure}

    \caption{Inputs and Outputs of the Spike FFN.}
    \label{fig:llama_token_other}
\end{figure*}

\subsection{Massive Activations}

Figure \ref{fig:llama1_top3} shows the block-wise differences of the top-3 channel magnitudes of the transformer block outputs $\boldsymbol{x}+\boldsymbol{u}+\boldsymbol{y}$ in LLaMA-3.2-1B. As depicted, the differences exhibit a single sharp jump at the second transformer block, where the spike FFN triggers massive activations. In the subsequent middle blocks, the differences remain close to zero, indicating that no further massive activations are triggered and that the residual stream alone preserves the existing ones until the late blocks.

\subsection{Input and Output of the Spike FFN for other Sample Sequences}
\label{appendix:other_sample}

Following Section \ref{sec:3}, we depict the activations of the other 4 randomly generated sequences in Figure \ref{fig:llama_token_other}.

\subsection{Pareto Inputs for the Spike FFN}
\label{appendix:a3a5}


\begin{figure*}[htbp]
    \centering
    \hfill
    \begin{subfigure}[b]{0.48\textwidth}
        \centering
        \includegraphics[width=\linewidth]{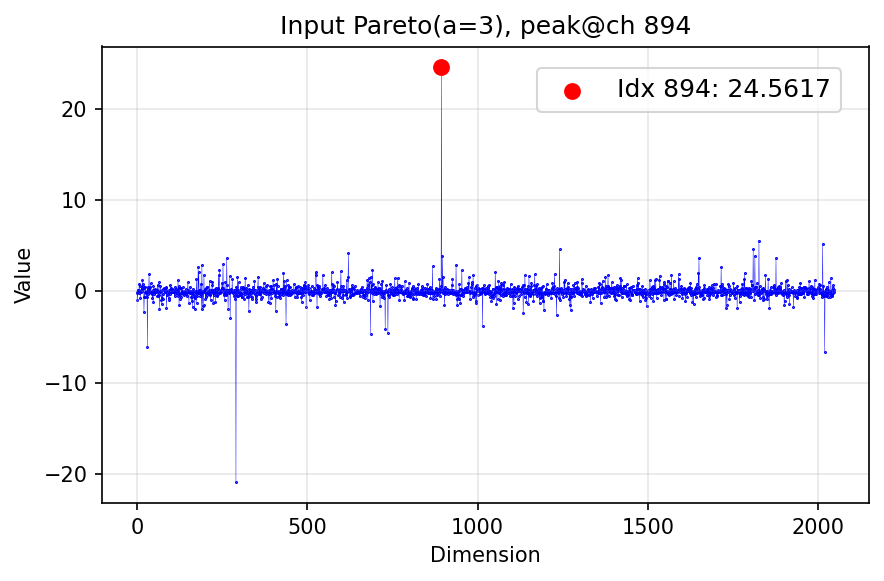}
    \end{subfigure}
    \hfill
    \begin{subfigure}[b]{0.48\textwidth}
        \centering
        \includegraphics[width=\linewidth]{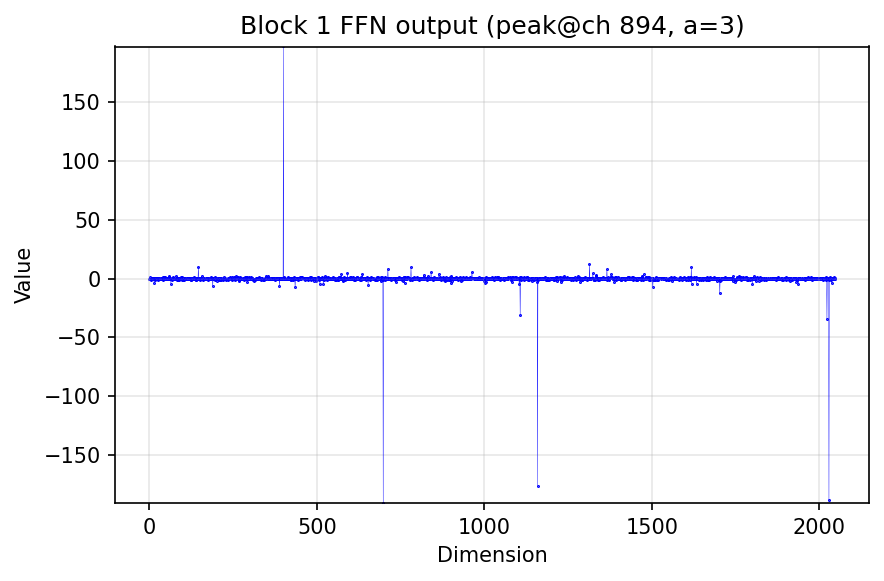}
    \end{subfigure}
    
    \hfill
    \begin{subfigure}[b]{0.48\textwidth}
        \centering
        \includegraphics[width=\linewidth]{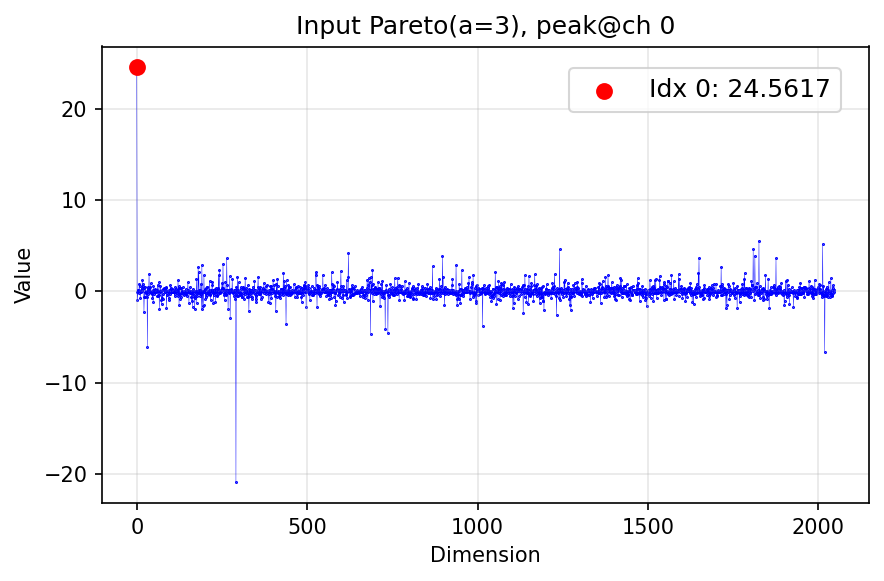}
    \end{subfigure}
    \hfill
    \begin{subfigure}[b]{0.48\textwidth}
        \centering
        \includegraphics[width=\linewidth]{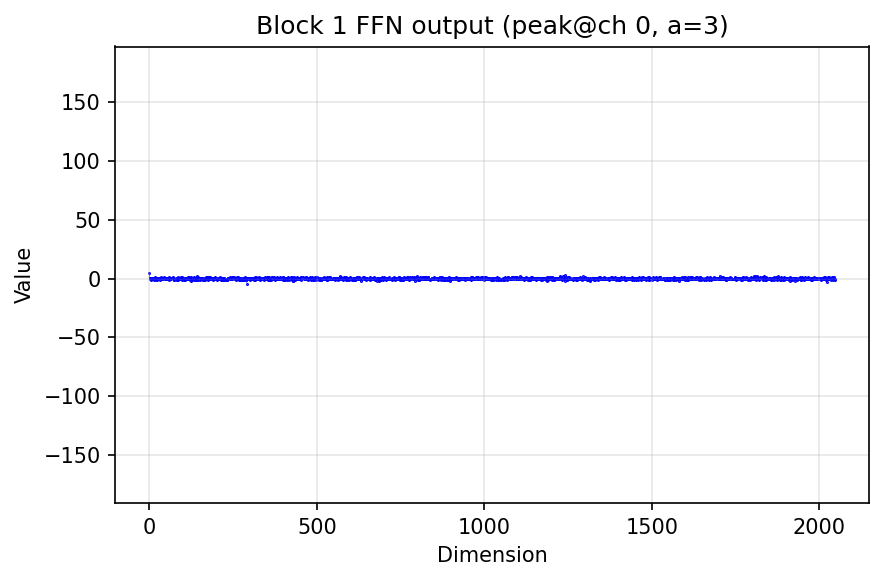}
    \end{subfigure}

    \hfill
    \begin{subfigure}[b]{0.48\textwidth}
        \centering
        \includegraphics[width=\linewidth]{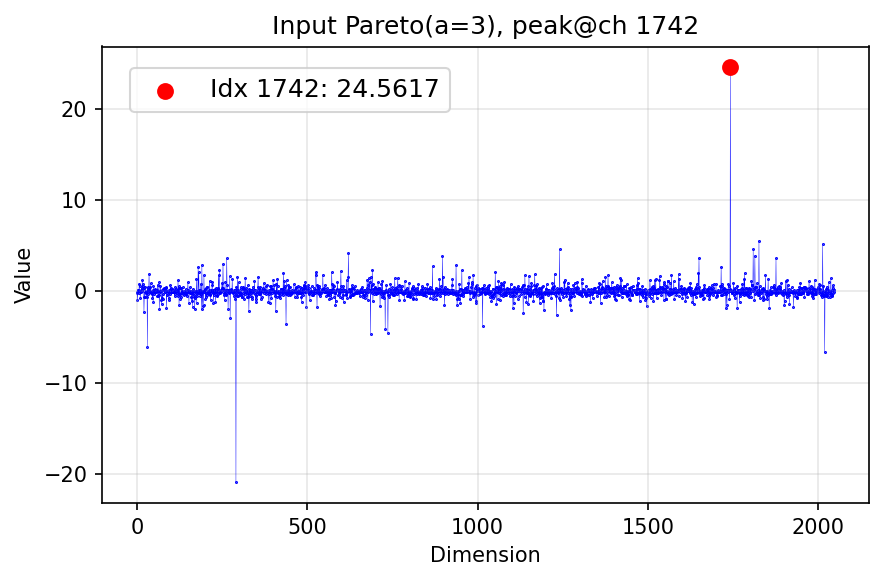}
    \end{subfigure}
    \hfill
    \begin{subfigure}[b]{0.48\textwidth}
        \centering
        \includegraphics[width=\linewidth]{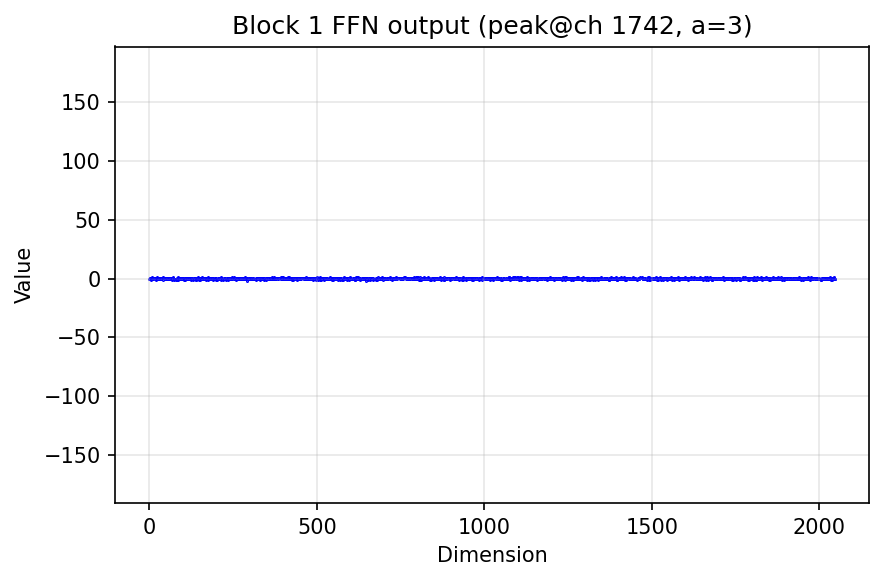}
    \end{subfigure}

    \caption{From top to bottom, the figure shows the output of the spike FFN when the input vector values are randomly sampled from a Pareto distribution ($a=3$), with the input value of the $894$th (MAGC), $0$th, or $1742$nd channel swapped with the maximum input value, respectively.}
    \label{fig:layer1FFN_a3}
\end{figure*}

\begin{figure*}[htb]
    \centering
    \hfill
    \begin{subfigure}[b]{0.48\textwidth}
        \centering
        \includegraphics[width=\linewidth]{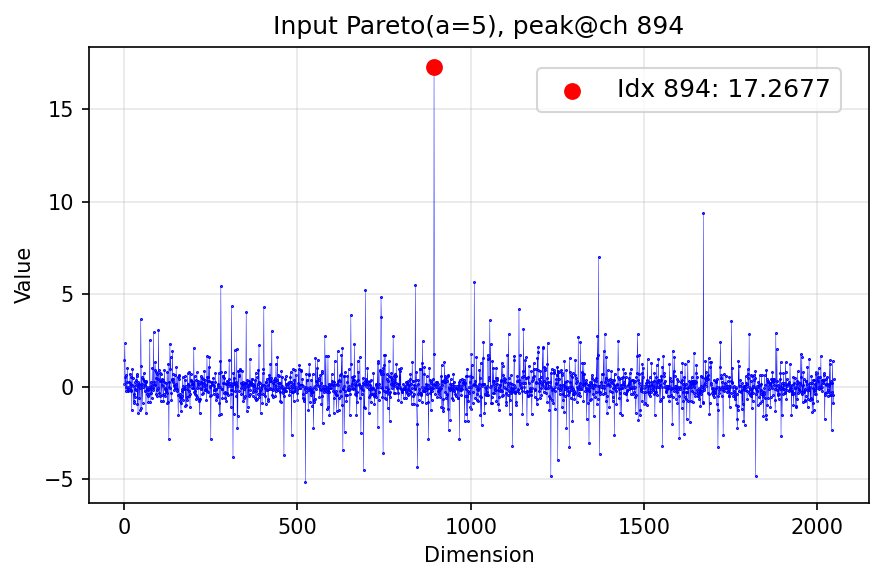}
    \end{subfigure}
    \hfill
    \begin{subfigure}[b]{0.48\textwidth}
        \centering
        \includegraphics[width=\linewidth]{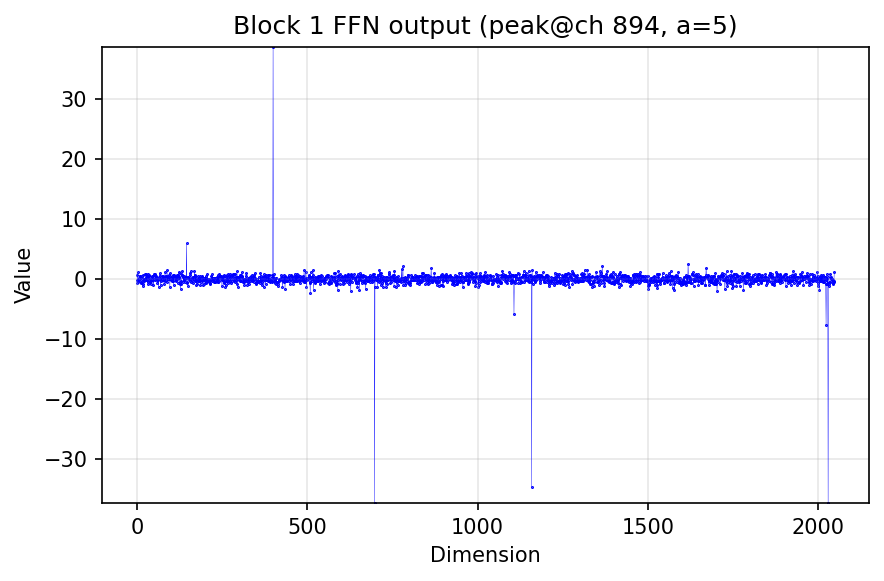}
    \end{subfigure}
    
    \hfill
    \begin{subfigure}[b]{0.48\textwidth}
        \centering
        \includegraphics[width=\linewidth]{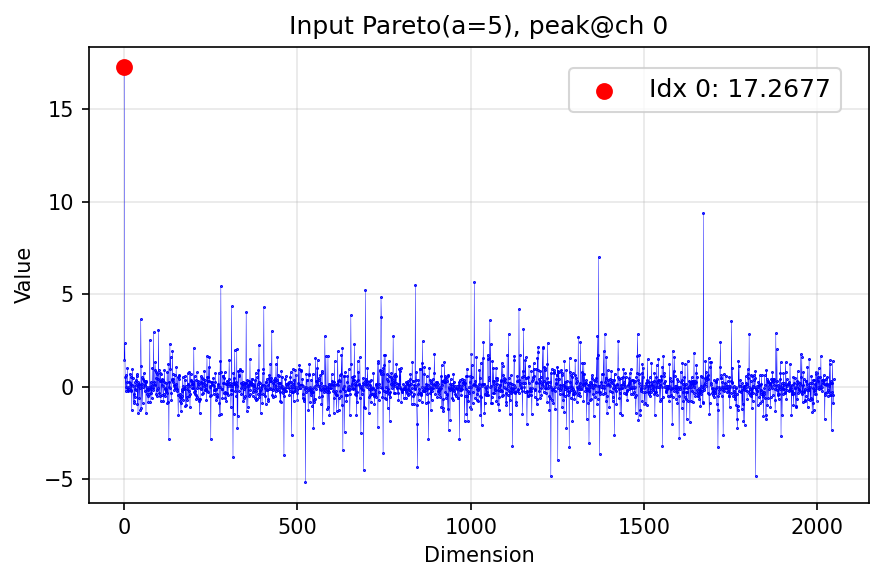}
    \end{subfigure}
    \hfill
    \begin{subfigure}[b]{0.48\textwidth}
        \centering
        \includegraphics[width=\linewidth]{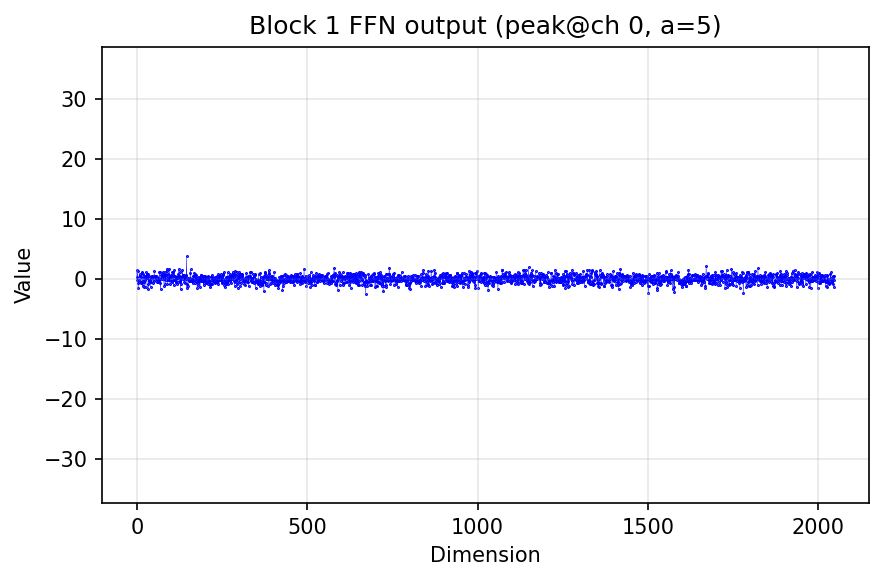}
    \end{subfigure}

    \hfill
    \begin{subfigure}[b]{0.48\textwidth}
        \centering
        \includegraphics[width=\linewidth]{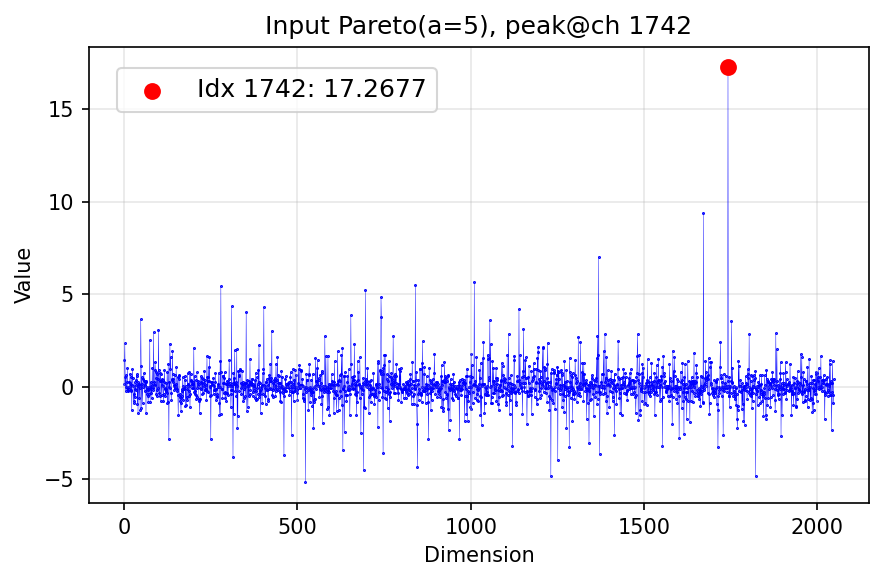}
    \end{subfigure}
    \hfill
    \begin{subfigure}[b]{0.48\textwidth}
        \centering
        \includegraphics[width=\linewidth]{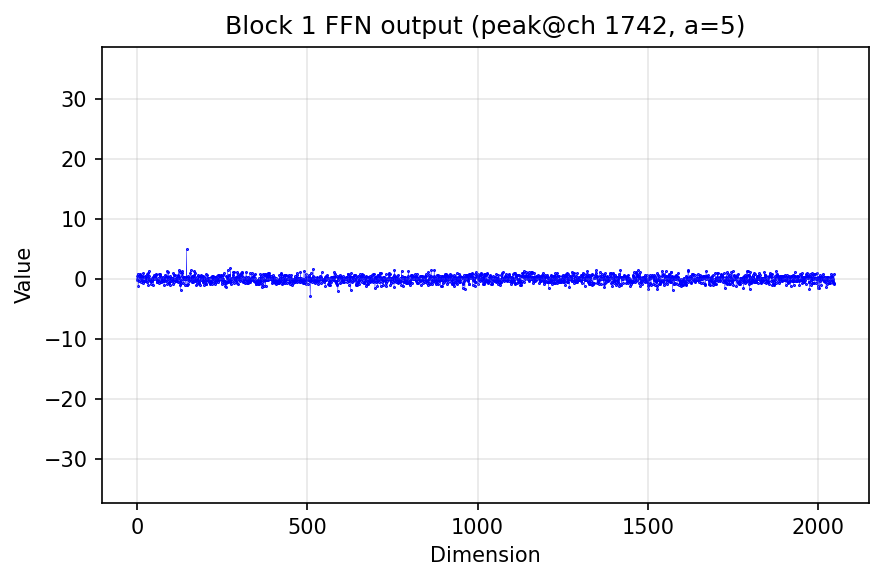}
    \end{subfigure}

    \caption{From top to bottom, the figure shows the output of the spike FFN when the input vector values are randomly sampled from a Pareto distribution ($a=5$), with the input value of the $894$th (MAGC), $0$th, or $1742$nd channel swapped with the maximum input value, respectively.}
    \label{fig:layer1FFN_a5}
\end{figure*}

\section{Proofs}

\subsection{Proof of Theorem 4.1}
\label{appendix:proof1}

\begin{theorem}
Given input $\boldsymbol{h}(x)=x \cdot \boldsymbol{e}_{i}+ \boldsymbol{\bar h}$.
When $x\to +\infty$, the output of an FFN has a quadratic form $\boldsymbol{y}(\boldsymbol{h}(x)) = x^2 \cdot \boldsymbol{y}^*+o(x^2)$ with constant $\boldsymbol{y}^*\in \mathbb{R}^d$ equals to 
\begin{align}
    \boldsymbol{y}^* = \sum_{j=1}^m \alpha_j\cdot \boldsymbol{\omega}_j,
\end{align}
where 
$\boldsymbol{\alpha} = (\alpha_1,\dots,\alpha_m)$ is the combination of coefficients in the weight matrices:
\begin{align}
    \alpha_j = \begin{cases}
        W_{\mathrm{gate}}^{(j,i)} W_{\mathrm{up}}^{(j,i)} \left\| \boldsymbol{W}_{\mathrm{down}}^{(:,j)} \right\|_2 & W_{\mathrm{gate}}^{(j,i)} > 0 \\
        0 & W_{\mathrm{gate}}^{(j,i)} \leq  0
    \end{cases}.
\end{align}
\end{theorem}

\begin{proof}
Based on Equation \ref{equation:hx}, the linear representation of $\boldsymbol{W}_{\mathrm{gate}}\boldsymbol{h}(x)$ and $\boldsymbol{W}_{\mathrm{up}}\boldsymbol{h}(x)$ can be written as 
\begin{align}
    \boldsymbol{W}_{\mathrm{gate}}\boldsymbol{h}(x) = x \boldsymbol{W}_{\mathrm{gate}}^{(:,i)} + \boldsymbol{b}_{\mathrm{gate}} \\
    \boldsymbol{W}_{\mathrm{up}}\boldsymbol{h}(x) = x \boldsymbol{W}_{\mathrm{up}}^{(:,i)} + \boldsymbol{b}_{\mathrm{up}}
\end{align}
where $\boldsymbol{b}_{\mathrm{gate}}=\sum_{i=1}^d \bar h_i \boldsymbol{W}_{\mathrm{gate}}^{(:,i)}$ and $\boldsymbol{b}_{\mathrm{up}}=\sum_{i=1}^d \bar h_i \boldsymbol{W}_{\mathrm{up}}^{(:,i)}$ are both constant under $x$.

We have the asymptotic property of $\mathrm{SiLU}(x)$: 
\begin{align}
    \lim_{x\to +\infty} \frac{\mathrm{SiLU}(\boldsymbol{W}_{\mathrm{gate}}h(x))}{x}
    = \begin{cases}
        W_{\mathrm{gate}}^{(j,i)} & W_{\mathrm{gate}}^{(j,i)}\geq 0 \\
        0 & W_{\mathrm{gate}}^{(j,i)}<0
    \end{cases}
\end{align}
and the asymptotic property of $\boldsymbol{W}_{\mathrm{up}}\boldsymbol{h}(x)$:
\begin{align}
    \lim_{x\to +\infty} \frac{\boldsymbol{W}_{\mathrm{up}}\boldsymbol{h}(x)}{x} = \boldsymbol{W}_{\mathrm{up}}^{(:,i)}
\end{align}

Finally,
\begin{align}
        \lim_{x\to +\infty} \frac{\boldsymbol{y}(\boldsymbol{h}(x))}{x^2} =
        \boldsymbol{W}_{\mathrm{down}} \lim_{x\to +\infty}\frac{\mathrm{SiLU}(\boldsymbol{W}_{\mathrm{gate}}h(x))}{x} \odot \lim_{x\to +\infty} \frac{\boldsymbol{W}_{\mathrm{up}}\boldsymbol{h}(x)}{x}=\sum_{j=1}^m \alpha_j \cdot \boldsymbol{\omega}_j,
    \end{align}
    where 
$\boldsymbol{\alpha} = (\alpha_1,\dots,\alpha_m)$ is 
\begin{align}
    \alpha_j = \begin{cases}
        W_{\mathrm{gate}}^{(j,i)} W_{\mathrm{up}}^{(j,i)} \left\| \boldsymbol{W}_{\mathrm{down}}^{(:,j)} \right\|_2 & W_{\mathrm{gate}}^{(j,i)} > 0 \\
        0 & W_{\mathrm{gate}}^{(j,i)} \leq  0
    \end{cases}.
\end{align}
\end{proof}

\subsection{Proof of Proposition~\ref{prop:other_channel}}
\label{appendix:proof2}

Recall that
\[
    \boldsymbol{y}^*
    =
    \sum_{j=1}^m \alpha_j \boldsymbol{\omega}_j,
    \qquad
    \|\boldsymbol{\omega}_j\|_2=1 .
\]
The energy distribution of $\boldsymbol{\alpha}$ is
\[
    p_j(\boldsymbol{\alpha})
    =
    \frac{\alpha_j^2}{\|\boldsymbol{\alpha}\|_2^2},
    \qquad j=1,\dots,m.
\]
Let \(j^* = \arg\max_{j} p_j(\boldsymbol{\alpha}) \)
be the dominant coefficient index. Then \(\boldsymbol{y}^*\) can be decomposed into the dominant column and the residual part:
\[
    \boldsymbol{y}^*
    =
    \alpha_{j^*}\boldsymbol{\omega}_{j^*}
    +
    \sum_{j\neq j^*}\alpha_j\boldsymbol{\omega}_j .
\]
Define the residual vector
\[
    \boldsymbol r
    =
    \sum_{j\neq j^*}\alpha_j\boldsymbol{\omega}_j .
\]
The residual amplification factor is defined as in Definition~\ref{def:kappa},
\[
    \kappa
    =
    \frac{\|\boldsymbol r\|_2}
    {\sqrt{\sum_{j\neq j^*}\alpha_j^2}}.
\]

The angular deviation between the true asymptotic output and the dominant column is 
\[
    |\sin\angle(\boldsymbol{y}^*,\boldsymbol{\omega}_{j^*})|
    =
    \sqrt{
    1-
    \left(
    \frac{
    \boldsymbol{y}^*\cdot \boldsymbol{\omega}_{j^*}
    }{
    \|\boldsymbol{y}^*\|_2\|\boldsymbol{\omega}_{j^*}\|_2
    }
    \right)^2
    } .
\]
Finally, for \(\tau\in(0,1-\frac1d)\), define
\[
    \Delta(\tau)
    =
    -\tau\log \tau
    -(1-\tau)\log(1-\tau)
    +\tau\log(d-1).
\]

\begin{lemma}
\label{lemma:sinangle}
Assume \( |\sin\angle(\boldsymbol{y}^*,\boldsymbol{\omega}_{j^*})| \leq 1-\frac{1}{d} \)  .
Then,
\[
\left|R(\boldsymbol{y}^*)-R(\boldsymbol{\omega}_{j^*})\right|
\leq
\Delta\left(|\sin\angle(\boldsymbol{y}^*,\boldsymbol{\omega}_{j^*})| \right).
\]
\end{lemma}
\begin{proof}
Based on the Fannes-Audenaert inequality:
\[
    |R(\boldsymbol{y}^*)-R(\boldsymbol{\omega}_{j^*})|=|H(\boldsymbol{\omega}_{j^*})-H(\boldsymbol{\omega}_{j^*})|
    \le
    \Delta(t),
\]
where $t:=\frac12\|p(\boldsymbol{y}^*)-p(\boldsymbol{\omega}_{j^*})\|_1$ and $H$ is the entropy function. We have
\[
\begin{aligned}
\frac12
\|p(\boldsymbol{y}^*)-p(\boldsymbol{\omega}_{j^*})\|_1
&=
\frac12
\sum_{i=1}^d
\left|
\left(
\frac{y_i^*}{\|\boldsymbol{y}^*\|_2}
\right)^2
-
\left(
\frac{\omega_{j^*,i}}{\|\boldsymbol{\omega}_{j^*}\|_2}
\right)^2
\right| \\
&=
\frac12
\sum_{i=1}^d
\left|
\left(
\frac{y_i^*}{\|\boldsymbol{y}^*\|_2}
-
\frac{\omega_{j^*,i}}{\|\boldsymbol{\omega}_{j^*}\|_2}
\right)
\left(
\frac{y_i^*}{\|\boldsymbol{y}^*\|_2}
+
\frac{\omega_{j^*,i}}{\|\boldsymbol{\omega}_{j^*}\|_2}
\right)
\right| \\
&\le
\frac12
\left[
\sum_{i=1}^d
\left(
\frac{y_i^*}{\|\boldsymbol{y}^*\|_2}
-
\frac{\omega_{j^*,i}}{\|\boldsymbol{\omega}_{j^*}\|_2}
\right)^2
\right]^{1/2}
\left[
\sum_{i=1}^d
\left(
\frac{y_i^*}{\|\boldsymbol{y}^*\|_2}
+
\frac{\omega_{j^*,i}}{\|\boldsymbol{\omega}_{j^*}\|_2}
\right)^2
\right]^{1/2} \\
&=
\frac12
\left[
2-
2\frac{
\boldsymbol{y}^*\cdot \boldsymbol{\omega}_{j^*}
}{
\|\boldsymbol{y}^*\|_2\|\boldsymbol{\omega}_{j^*}\|_2
}
\right]^{1/2}
\left[
2+
2\frac{
\boldsymbol{y}^*\cdot \boldsymbol{\omega}_{j^*}
}{
\|\boldsymbol{y}^*\|_2\|\boldsymbol{\omega}_{j^*}\|_2
}
\right]^{1/2} \\
&=
\left[
1-
\left(
\frac{
\boldsymbol{y}^*\cdot \boldsymbol{\omega}_{j^*}
}{
\|\boldsymbol{y}^*\|_2\|\boldsymbol{\omega}_{j^*}\|_2
}
\right)^2
\right]^{1/2} \\
&=
|\sin\angle(\boldsymbol{y}^*,\boldsymbol{\omega}_{j^*})|.
\end{aligned}
\]

Since \(\Delta(t)\) is increasing on \(0\le t\le 1-\frac1d\), 
\[
    \left|R(\boldsymbol{y}^*)-R(\boldsymbol{\omega}_{j^*})\right|
    \leq
    \Delta \left(
    \min\left\{|\sin\angle(\boldsymbol{y}^*,\boldsymbol{\omega}_{j^*})|,1-\frac1d\right\}
    \right)=\Delta \left(
    |\sin\angle(\boldsymbol{y}^*,\boldsymbol{\omega}_{j^*})|
    \right).
\]
\end{proof}

\begin{proposition} 
Let 
\begin{align}
   \tau=\kappa \sqrt{\frac{1-p_{j^*}(\boldsymbol{\alpha})}{p_{j^*}(\boldsymbol{\alpha})}}.
\end{align}
We have
\begin{align}
    \left|R(\boldsymbol{y}^*)-R(\boldsymbol{\omega}_{j^*})\right|
\leq
\Delta\left(
\tau
\right),
\end{align}
where for $\tau \in\left(0,1-\frac{1}{d}\right)$,
$
\Delta(\tau)
=
-\tau\log \tau-(1-\tau)\log(1-\tau)+\tau\log(d-1).
$
\end{proposition}

\begin{proof} 
Recall $
    \boldsymbol r
    =
    \sum_{j\neq j^*}\alpha_j\boldsymbol{\omega}_j ,
$
then
$
    \boldsymbol y^*
    =
    \alpha_{j^*}\boldsymbol{\omega}_{j^*}
    +
    \boldsymbol r .
$
Since \(\|\boldsymbol{\omega}_{j^*}\|_2=1\), decompose the residual vector as
\[
    \boldsymbol r
    =
    r_{\parallel}\boldsymbol{\omega}_{j^*}
    +
    \boldsymbol r_{\perp},
    \qquad
    \boldsymbol r_{\perp}\perp \boldsymbol{\omega}_{j^*}.
\]
We have
\[
    \boldsymbol y^*
    =
    (\alpha_{j^*}+r_{\parallel})\boldsymbol{\omega}_{j^*}
    +
    \boldsymbol r_{\perp}.
\]
The angular deviation from \(\boldsymbol{\omega}_{j^*}\) is caused only by the perpendicular residual component. Hence,
\[
\begin{aligned}
|\sin\angle(\boldsymbol y^*,\boldsymbol{\omega}_{j^*})|
&=
\frac{
\|\boldsymbol r_{\perp}\|_2
}{
\left[
(\alpha_{j^*}+r_{\parallel})^2
+
\|\boldsymbol r_{\perp}\|_2^2
\right]^{1/2}
} \\
&\le
\frac{
\sqrt{r_{\parallel}^2+\|\boldsymbol r_{\perp}\|_2^2}
}{
|\alpha_{j^*}|
}
=
\frac{\|\boldsymbol r\|_2}{|\alpha_{j^*}|}.
\end{aligned}
\]
The above inequality follows because, after squaring and rearranging, it is equivalent to
\[
    0
    \le
    \left(
    r_{\parallel}(\alpha_{j^*}+r_{\parallel})
    +
    \|\boldsymbol r_{\perp}\|_2^2
    \right)^2 .
\]
Moreover,
\[
\frac{\|\boldsymbol r\|_2}{|\alpha_{j^*}|}=
\kappa
\frac{
\sqrt{
\sum_{j\neq j^*}\alpha_j^2
}
}{
|\alpha_{j^*}|
}=\kappa \sqrt{\frac{
    1-p_{j^*}(\boldsymbol\alpha)
    }{
    p_{j^*}(\boldsymbol\alpha)
    }}.
\]

Therefore,
\[
|\sin\angle(\boldsymbol y^*,\boldsymbol{\omega}_{j^*})|
\le
\kappa
\sqrt{
\frac{
1-p_{j^*}(\boldsymbol\alpha)
}{
p_{j^*}(\boldsymbol\alpha)
}
}.
\]

Based on Lemma~\ref{lemma:sinangle}, it is sufficient to prove Proposition \ref{prop:other_channel}. 
\end{proof}

\section{Additional Results}
\label{appendix:results}

\textbf{Identification of the Spike FFN and MAGC.}
We identify the spike FFN and MAGC by comparing the first (spike) token and the second (non-spike) token of the randomly generated sequences. Let $\boldsymbol{y}_0^l$ denote the average FFN output of the first token in the $l$-th transformer block. We regard an FFN as the spike FFN if its output exhibits an outlier channel \citep{sun2024massive}, i.e.,
\begin{align}
\label{eq:spike_ffn_criterion}
    \max_j \left|y_0^l(j)\right| > 100
    \quad \text{and} \quad
    \max_j \left|y_0^l(j)\right| > 10 \cdot \operatorname{median}_j \left|y_0^l(j)\right|. 
\end{align}

Given the spike FFN, we then identify MAGC from its inputs. Let
$\boldsymbol{h}_0$ and $\boldsymbol{h}_1$ be the average inputs to the spike FFN for the first and second tokens, respectively. The gating channel should be the channel where the two inputs deviate the most. We therefore locate MAGC as
\begin{align}
\label{eq:magc_criterion}
    i^* = \arg\max_j \left| h_0(j) - h_1(j) \right|,
\end{align}
and determine the gating direction by the sign of $h_0(i^*)-h_1(i^*)$: a positive (negative) sign indicates that massive activations emerge when the value of MAGC becomes sufficiently large (small). 

This appendix reports the detailed results for the models summarized in Table \ref{tab:model_sum}: LLaMA-2-7B (Figure \ref{fig:llama2}), Qwen2.5-1.5B (Figure \ref{fig:qwen}), QwQ-32B (Figure \ref{fig:qwq}), Ministral-3-3B-Instruct-2512 (Figure \ref{fig:ministral}), and DeepSeek-R1-Distill-Llama-8B (Figure \ref{fig:deepseek-llama}). For each model, we present (i) the input $\boldsymbol{h}$ and outputs $\boldsymbol{y}$ of the spike FFN for a random input sequence, (ii) the coefficient vector $\boldsymbol{\alpha}$ defined in Theorem~4.1, where for models with a negative gating direction $\boldsymbol{\alpha}$ is computed under the limit $x\to -\infty$ and the condition on $W_{\mathrm{gate}}^{(j,i)}$ flips sign accordingly, and (iii) the dominant columns (the columns most aligned with the output $\boldsymbol{y}_0$, as reported in Table~\ref{tab:model_sum}) together with a randomly selected column for comparison.

\newpage 

\begin{figure}[p]
    \centering
    
    \hfill
    \begin{subfigure}[b]{0.48\textwidth}
        \centering
        \includegraphics[width=\linewidth]{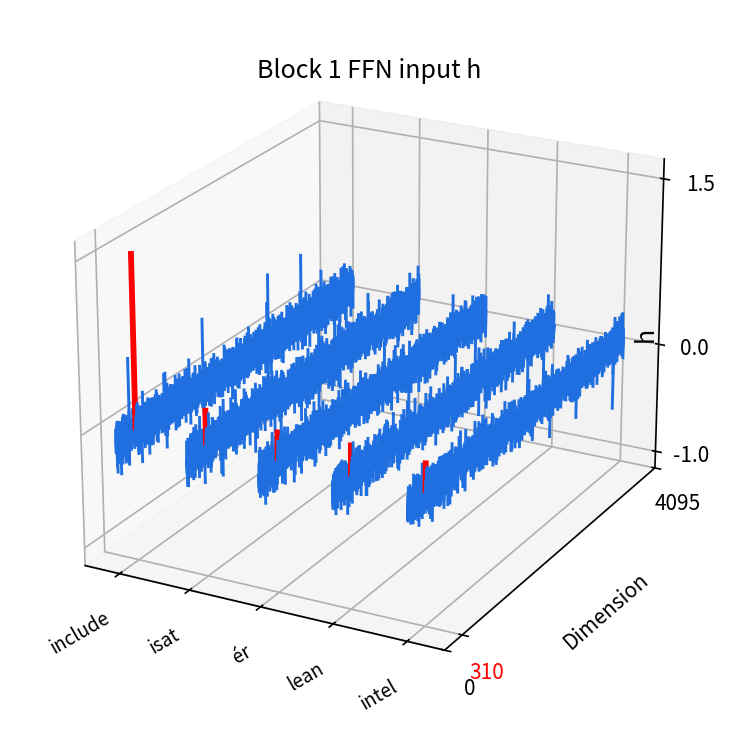}
        \caption{Input $\boldsymbol{h}$ to FFN.}
    \end{subfigure}
    \hfill
    \begin{subfigure}[b]{0.48\textwidth}
        \centering
        \includegraphics[width=\linewidth]{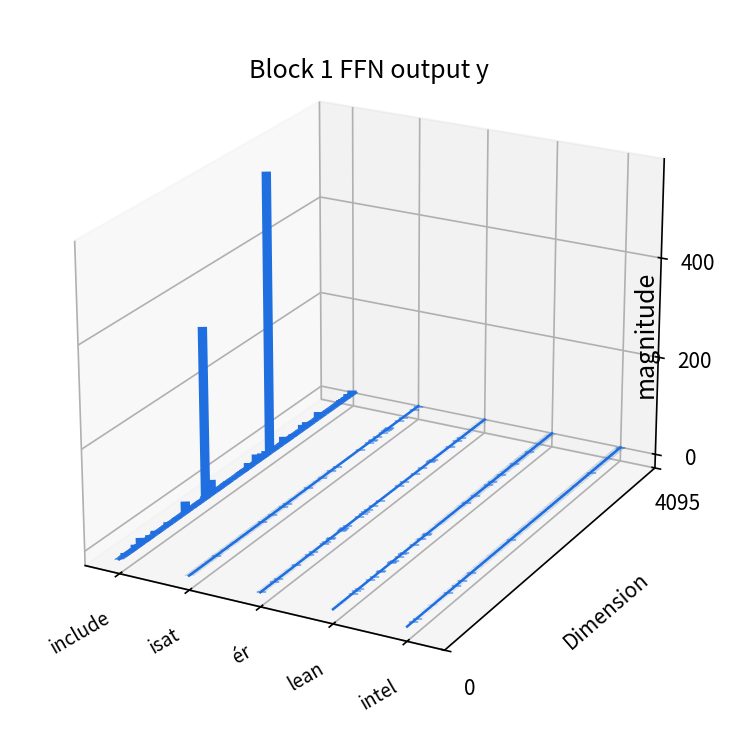}
        \caption{Output $\boldsymbol{y}$ of FFN.}
    \end{subfigure}
    

    \hfill
    \begin{subfigure}[b]{0.48\textwidth}
        \centering
        \includegraphics[width=\linewidth]{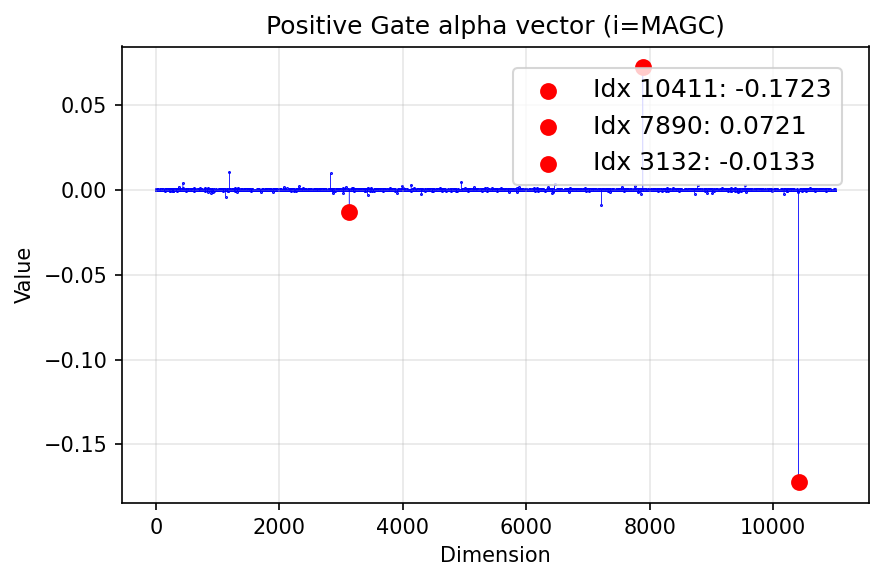}
        \caption{$\alpha$ vector when $i^*=310$.}
    \end{subfigure}
    \hfill
    \begin{subfigure}[b]{0.48\textwidth}
        \centering
        \includegraphics[width=\linewidth]{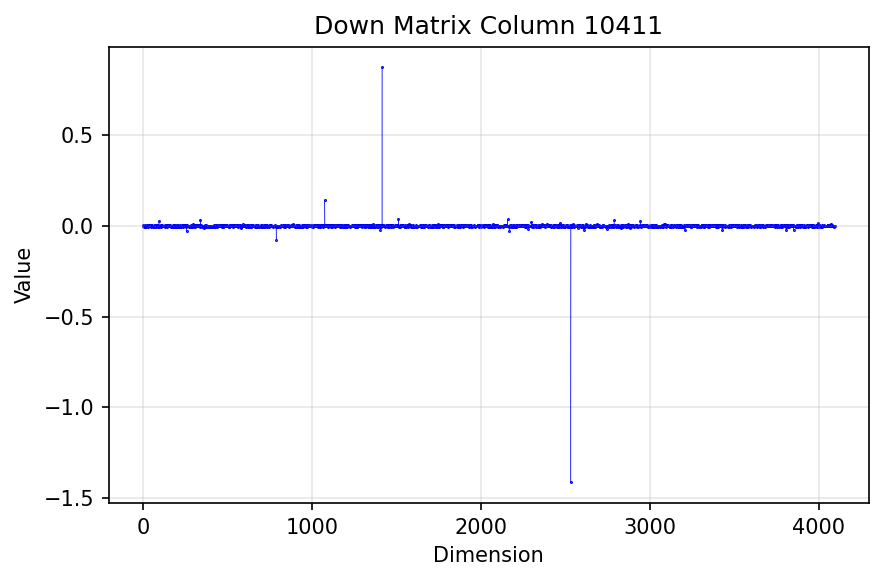}
        \caption{$\boldsymbol{\omega}_{10411}$.}
    \end{subfigure}

    \hfill
    \begin{subfigure}[b]{0.48\textwidth}
        \centering
        \includegraphics[width=\linewidth]{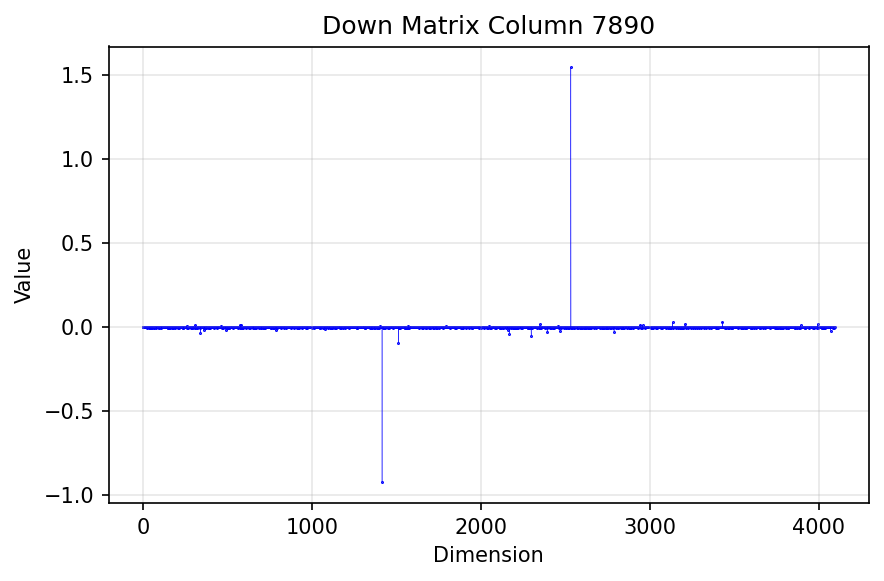}
        \caption{$\boldsymbol{\omega}_{7890}$.}
    \end{subfigure}
    \hfill
    \begin{subfigure}[b]{0.48\textwidth}
        \centering
        \includegraphics[width=\linewidth]{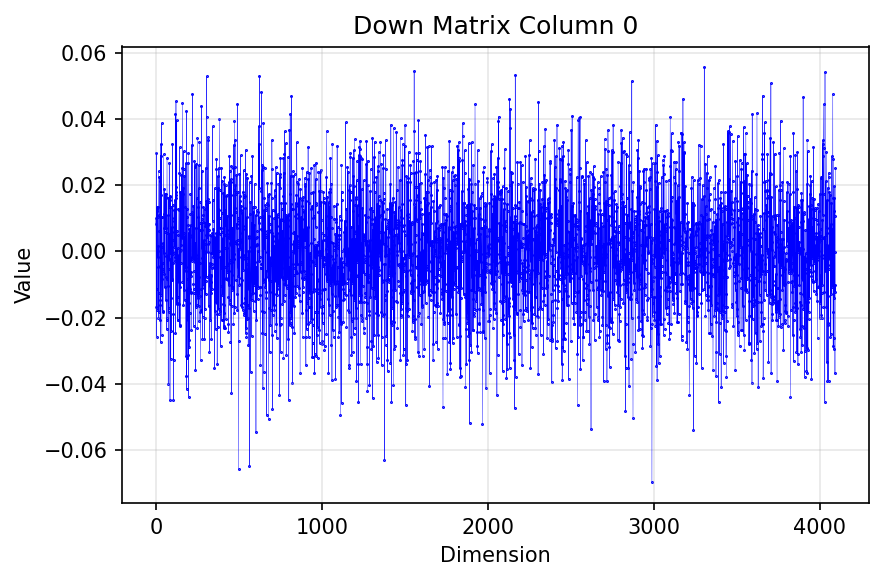}
        \caption{$\boldsymbol{\omega}_{0}$.}
    \end{subfigure}

    \caption{Results for LLaMA-2-7b-hf.}
    \label{fig:llama2}
\end{figure}


\begin{figure}[p]
    \centering
    

    \hfill
    \begin{subfigure}[b]{0.48\textwidth}
        \centering
        \includegraphics[width=\linewidth]{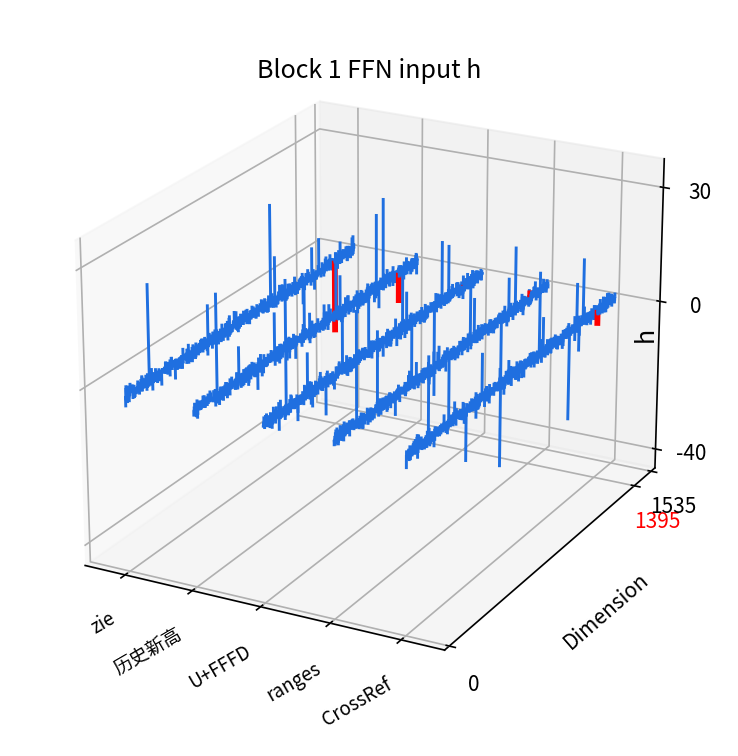}
        \caption{Input $\boldsymbol{h}$ to FFN.}
    \end{subfigure}
    \hfill
    \begin{subfigure}[b]{0.48\textwidth}
        \centering
        \includegraphics[width=\linewidth]{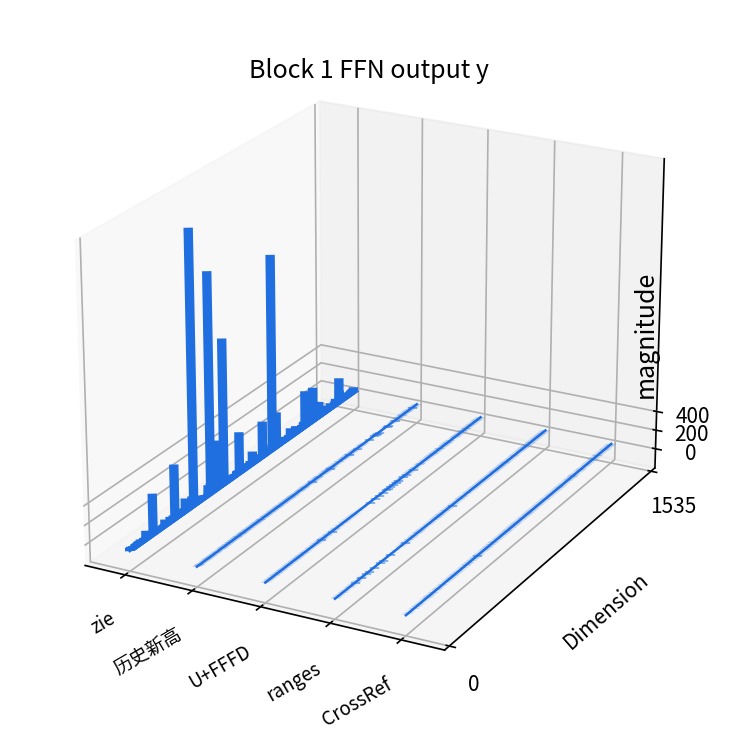}
        \caption{Output $\boldsymbol{y}$ of FFN.}
    \end{subfigure}

    \hfill
    \begin{subfigure}[b]{0.48\textwidth}
        \centering
        \includegraphics[width=\linewidth]{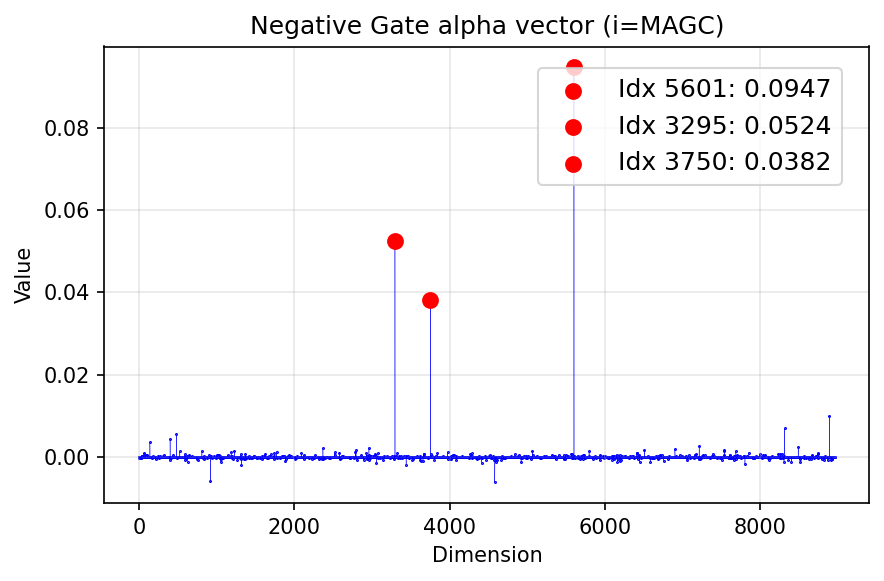}
        \caption{$\alpha$ vector when $i^*=1395$.}
    \end{subfigure}
    \hfill
    \begin{subfigure}[b]{0.48\textwidth}
        \centering
        \includegraphics[width=\linewidth]{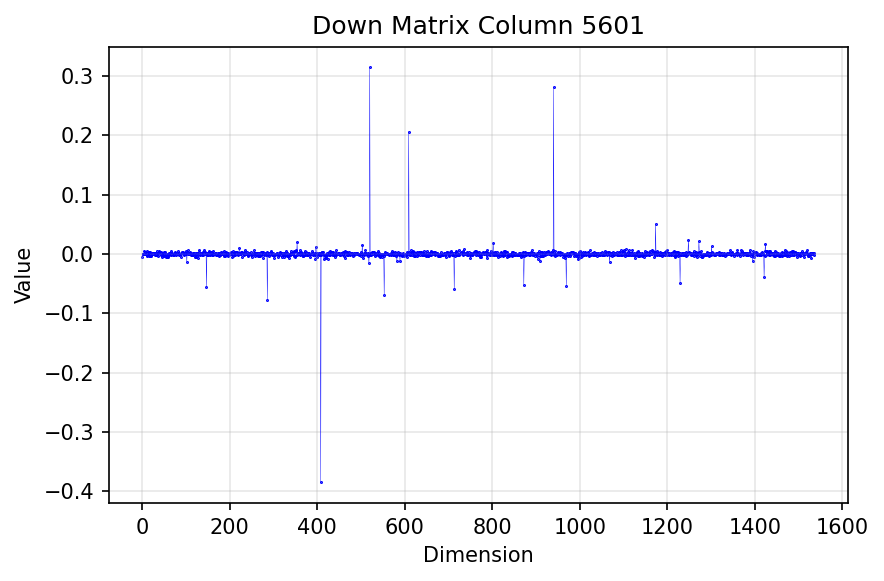}
        \caption{$\boldsymbol{\omega}_{5601}$.}
    \end{subfigure}

    \hfill
    \begin{subfigure}[b]{0.48\textwidth}
        \centering
        \includegraphics[width=\linewidth]{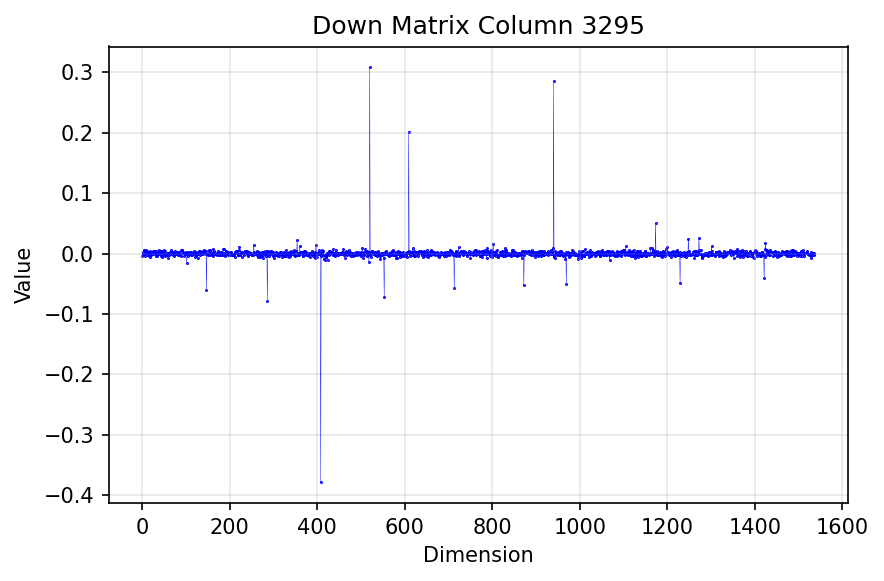}
        \caption{$\boldsymbol{\omega}_{3295}$.}
    \end{subfigure}
    \hfill
    \begin{subfigure}[b]{0.48\textwidth}
        \centering
        \includegraphics[width=\linewidth]{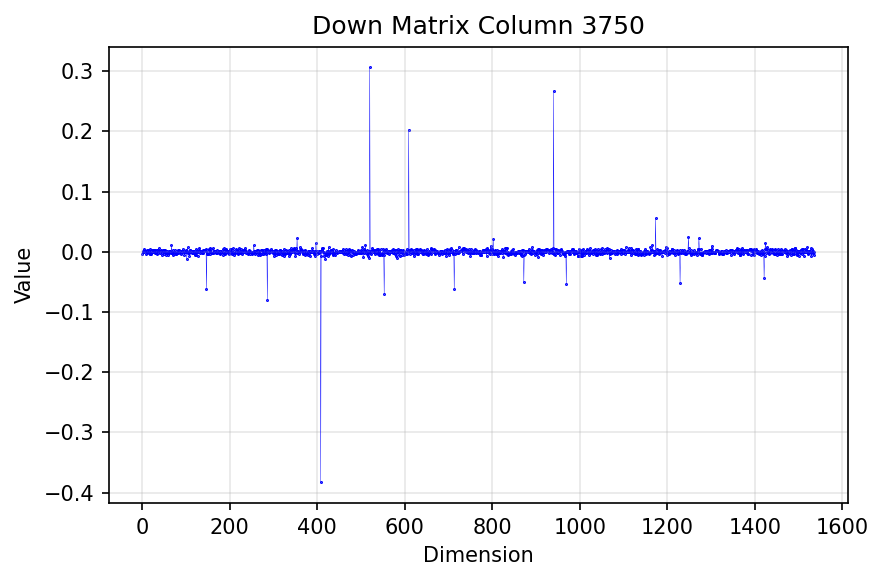}
        \caption{$\boldsymbol{\omega}_{3750}$.}
    \end{subfigure}

    \caption{Results for Qwen2.5-1.5b.}
    \label{fig:qwen}
\end{figure}

\begin{figure}[p]
    \centering
    

    \hfill
    \begin{subfigure}[b]{0.48\textwidth}
        \centering
        \includegraphics[width=\linewidth]{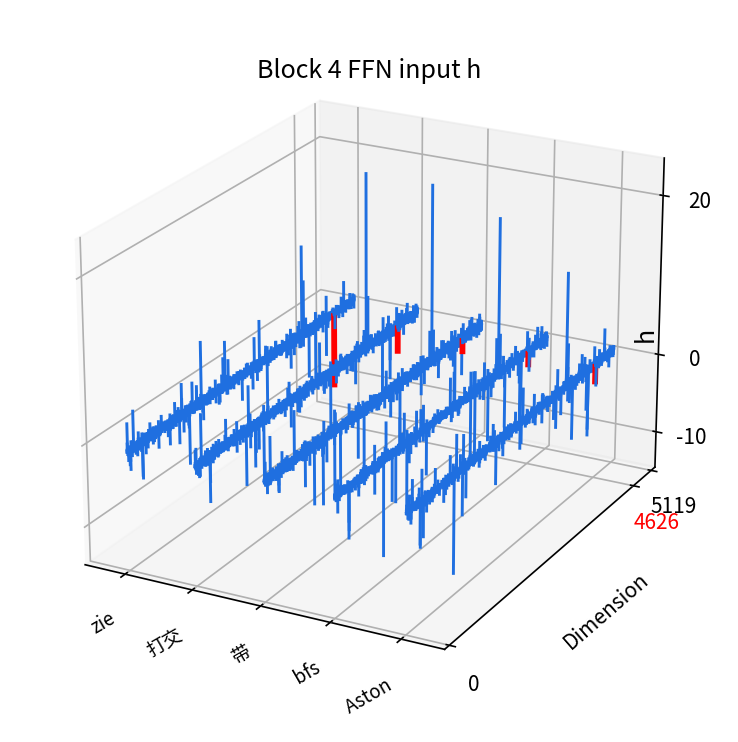}
        \caption{Input $\boldsymbol{h}$ to FFN.}
    \end{subfigure}
    \hfill
    \begin{subfigure}[b]{0.48\textwidth}
        \centering
        \includegraphics[width=\linewidth]{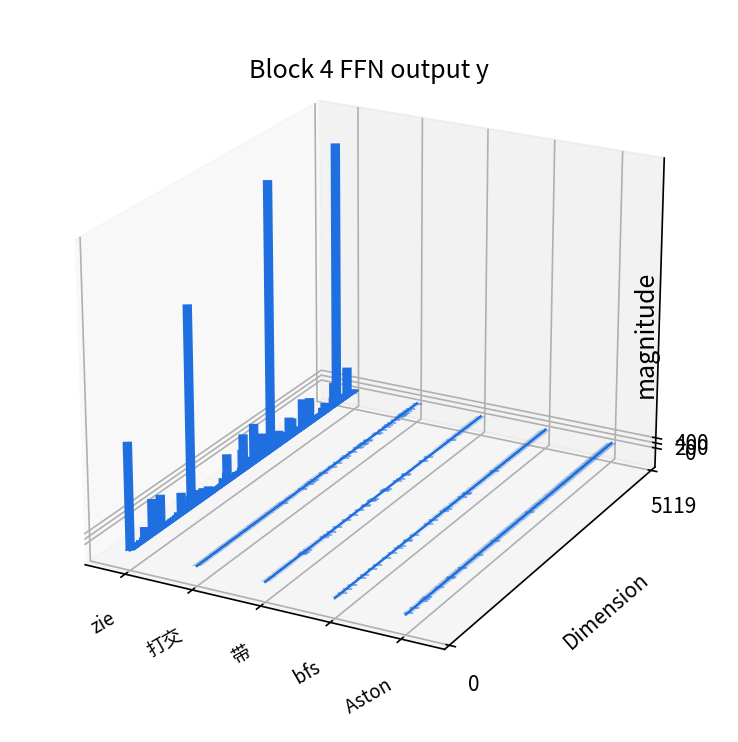}
        \caption{Output $\boldsymbol{y}$ of FFN.}
    \end{subfigure}

    \hfill
    \begin{subfigure}[b]{0.48\textwidth}
        \centering
        \includegraphics[width=\linewidth]{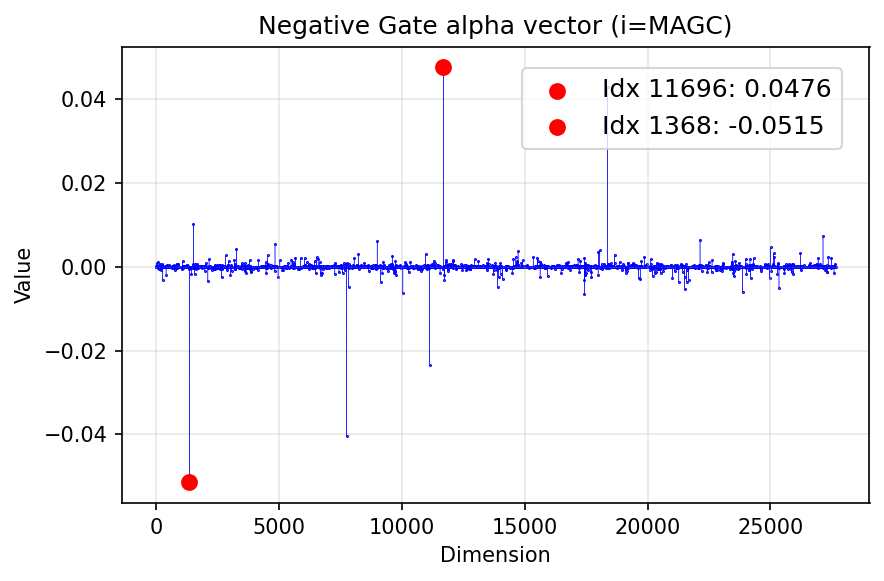}
        \caption{$\alpha$ vector when $i^*=4626$.}
    \end{subfigure}
    \hfill
    \begin{subfigure}[b]{0.48\textwidth}
        \centering
        \includegraphics[width=\linewidth]{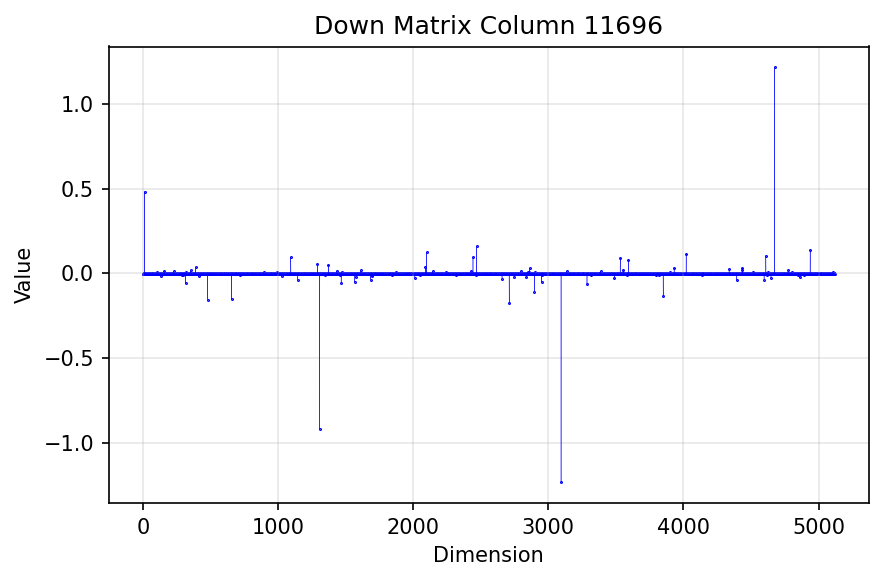}
        \caption{$\boldsymbol{\omega}_{11696}$.}
    \end{subfigure}

    \hfill
    \begin{subfigure}[b]{0.48\textwidth}
        \centering
        \includegraphics[width=\linewidth]{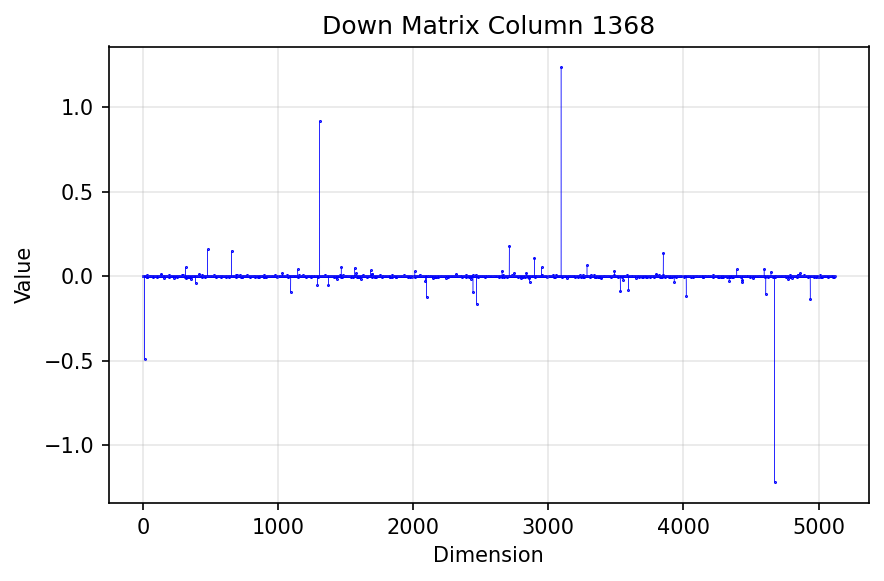}
        \caption{$\boldsymbol{\omega}_{1368}$.}
    \end{subfigure}
    \hfill
    \begin{subfigure}[b]{0.48\textwidth}
        \centering
        \includegraphics[width=\linewidth]{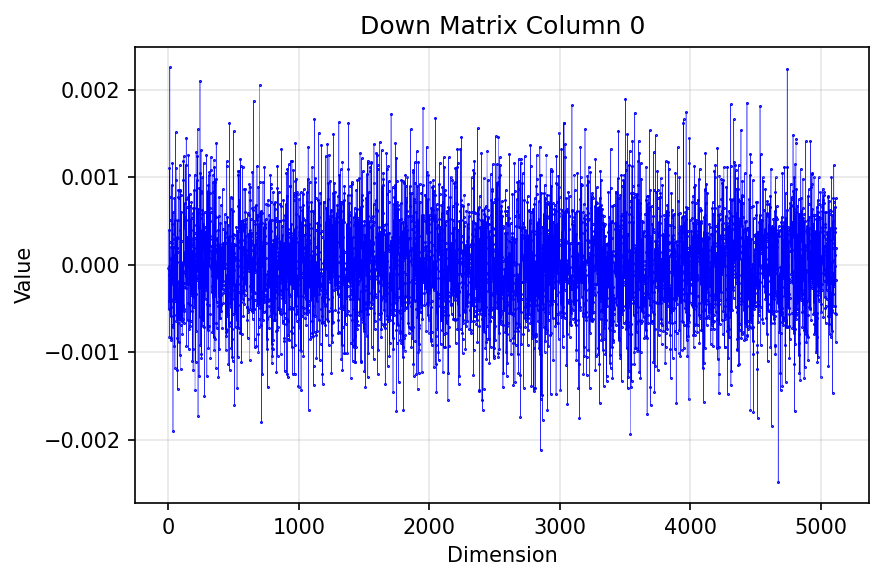}
        \caption{$\boldsymbol{\omega}_{0}$.}
    \end{subfigure}
    
    \caption{Results for QwQ-32B. For QwQ-32B, several of the channels with the largest input difference $|h_0(j)-h_1(j)|$ coincide with massive-activation channels propagated through the residual stream; we therefore identify the MAGC ($i^*=4626$, negative gating direction) by traversing the ten channels with the largest such difference following Equation~\ref{eq:magc_criterion} and selecting the first one that is not itself a massive-activation channel, rather than simply taking the single largest. }

    \label{fig:qwq}
\end{figure}

\begin{figure}[p]
    \centering
    

    \hfill
    \begin{subfigure}[b]{0.48\textwidth}
        \centering
        \includegraphics[width=\linewidth]{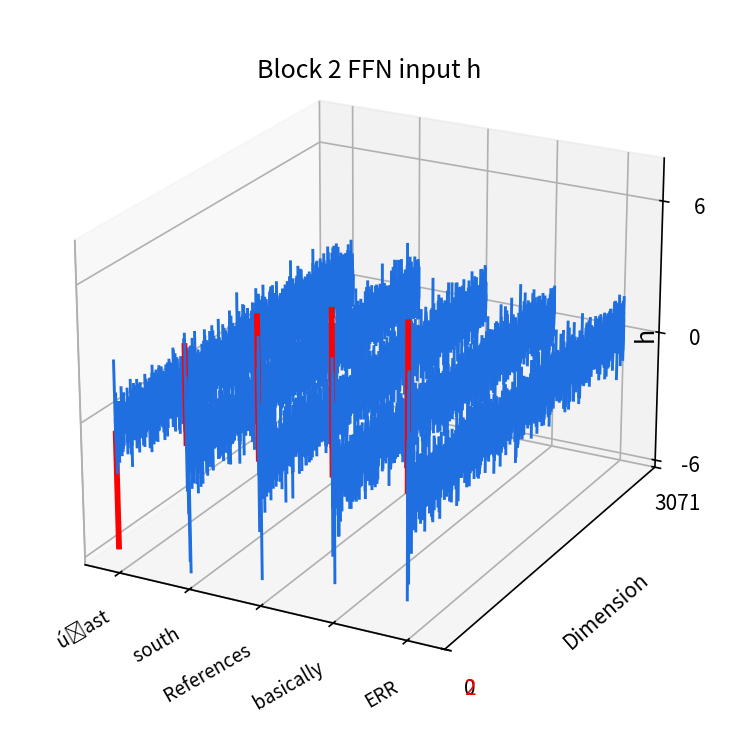}
        \caption{Input $\boldsymbol{h}$ to FFN.}
    \end{subfigure}
    \hfill
    \begin{subfigure}[b]{0.48\textwidth}
        \centering
        \includegraphics[width=\linewidth]{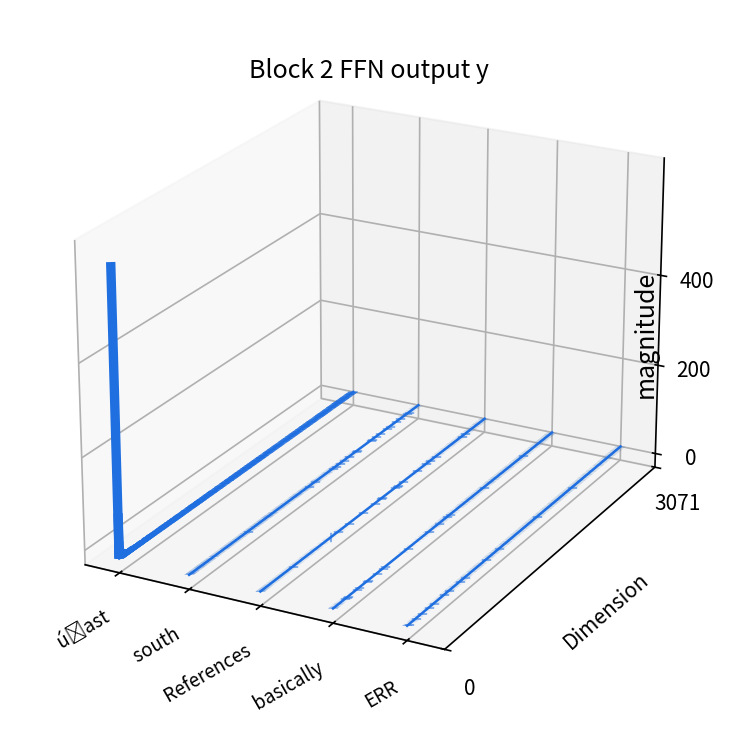}
        \caption{Output $\boldsymbol{y}$ of FFN.}
    \end{subfigure}

    \hfill
    \begin{subfigure}[b]{0.48\textwidth}
        \centering
        \includegraphics[width=\linewidth]{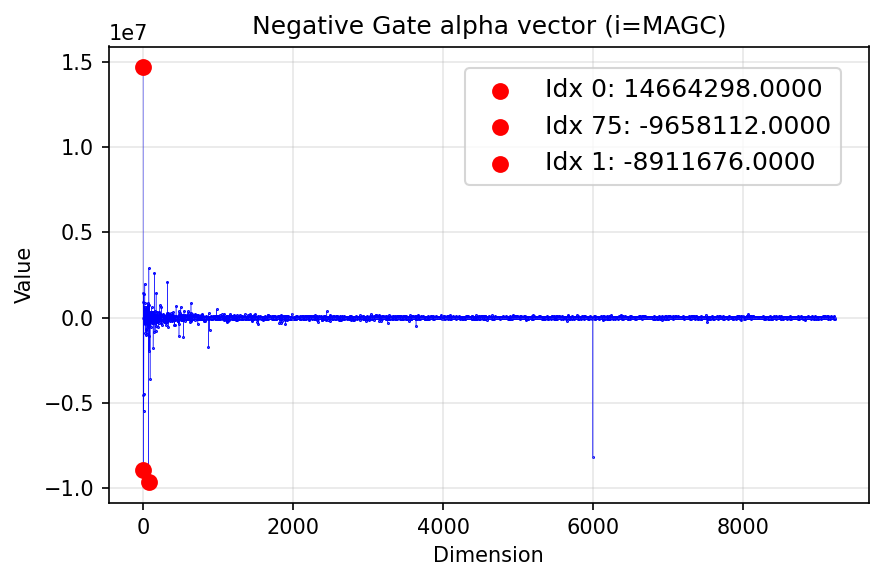}
        \caption{$\alpha$ vector when $i^*=2$.}
    \end{subfigure}
    \hfill
    \begin{subfigure}[b]{0.48\textwidth}
        \centering
        \includegraphics[width=\linewidth]{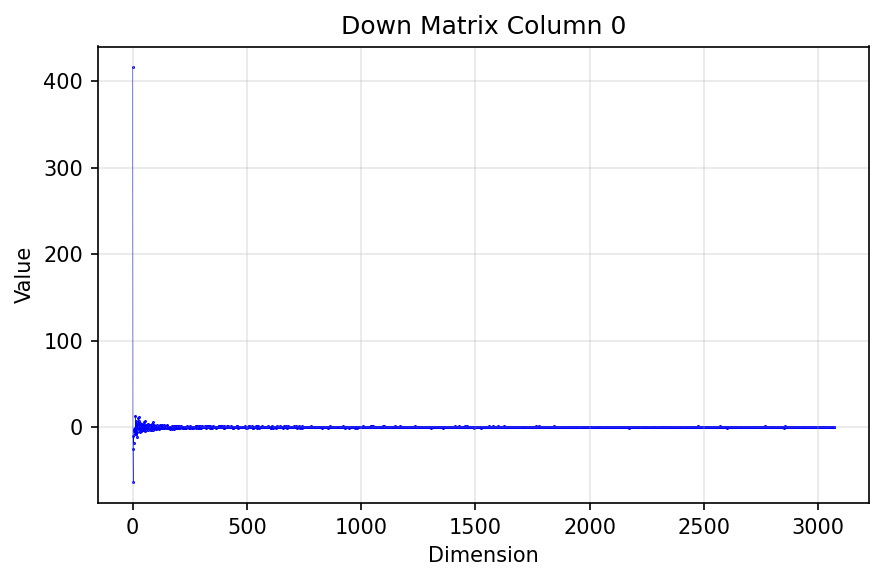}
        \caption{$\boldsymbol{\omega}_{0}$.}
    \end{subfigure}

    \hfill
    \begin{subfigure}[b]{0.48\textwidth}
        \centering
        \includegraphics[width=\linewidth]{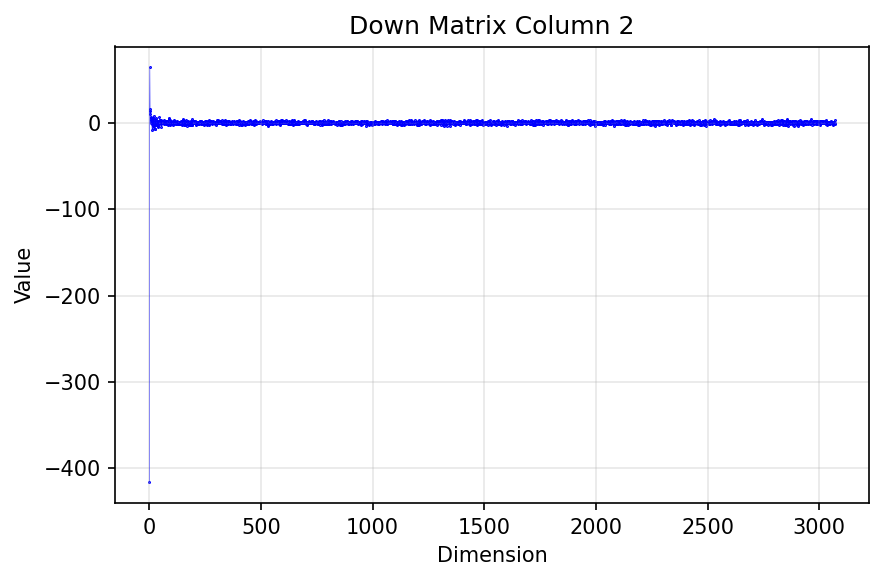}
        \caption{$\boldsymbol{\omega}_{2}$.}
    \end{subfigure}
    \hfill
    \begin{subfigure}[b]{0.48\textwidth}
        \centering
        \includegraphics[width=\linewidth]{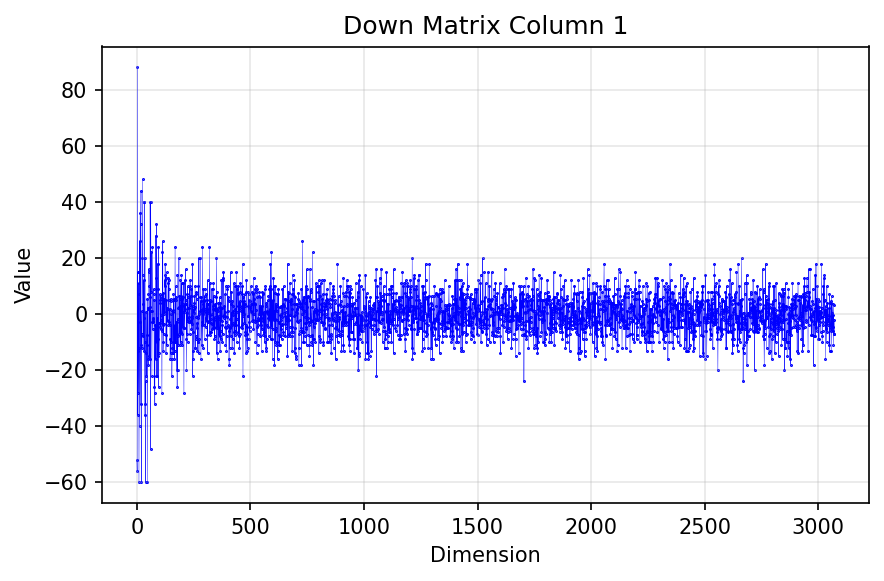}
        \caption{$\boldsymbol{\omega}_{1}$.}
    \end{subfigure}
    
    \caption{Results for Ministral-3-3B-Instruct-2512. Note that columns 0 and 2 belong to $j^*$. Instead of choosing 0 as the random column for comparison, we select column 1. }
    \label{fig:ministral}
\end{figure}

\begin{figure}[p]
    \centering
    

    \hfill
    \begin{subfigure}[b]{0.48\textwidth}
        \centering
        \includegraphics[width=\linewidth]{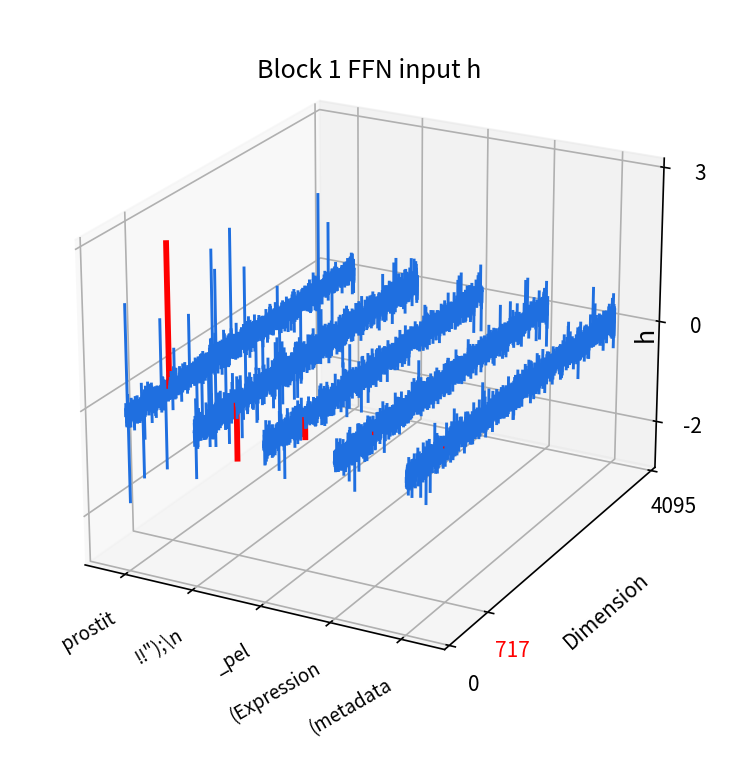}
        \caption{Input $\boldsymbol{h}$ to FFN.}
    \end{subfigure}
    \hfill
    \begin{subfigure}[b]{0.48\textwidth}
        \centering
        \includegraphics[width=\linewidth]{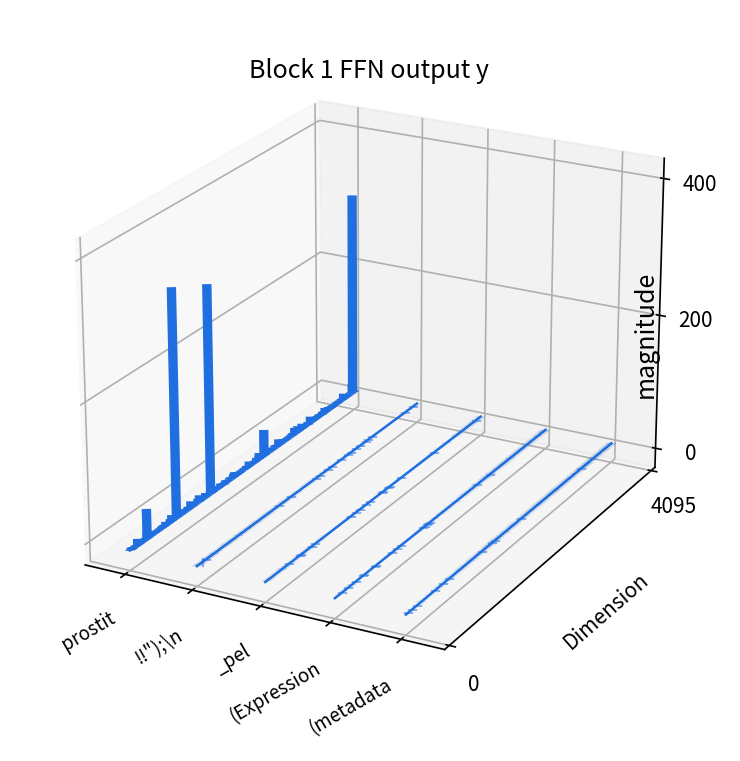}
        \caption{Output $\boldsymbol{y}$ of FFN.}
    \end{subfigure}

    \hfill
    \begin{subfigure}[b]{0.48\textwidth}
        \centering
        \includegraphics[width=\linewidth]{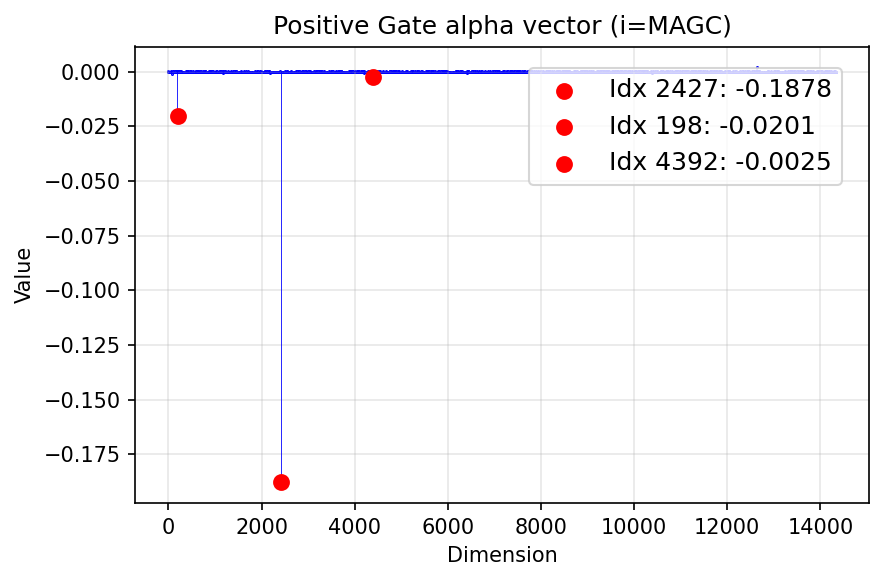}
        \caption{$\alpha$ vector when $i^*=717$.}
    \end{subfigure}
    \hfill
    \begin{subfigure}[b]{0.48\textwidth}
        \centering
        \includegraphics[width=\linewidth]{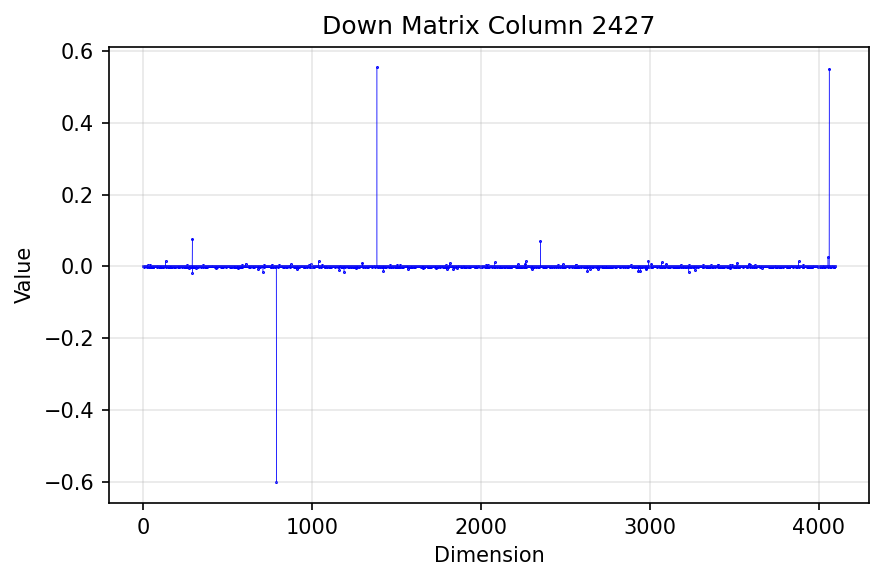}
        \caption{$\boldsymbol{\omega}_{2427}$.}
    \end{subfigure}

    \hfill
    \begin{subfigure}[b]{0.48\textwidth}
        \centering
        \includegraphics[width=\linewidth]{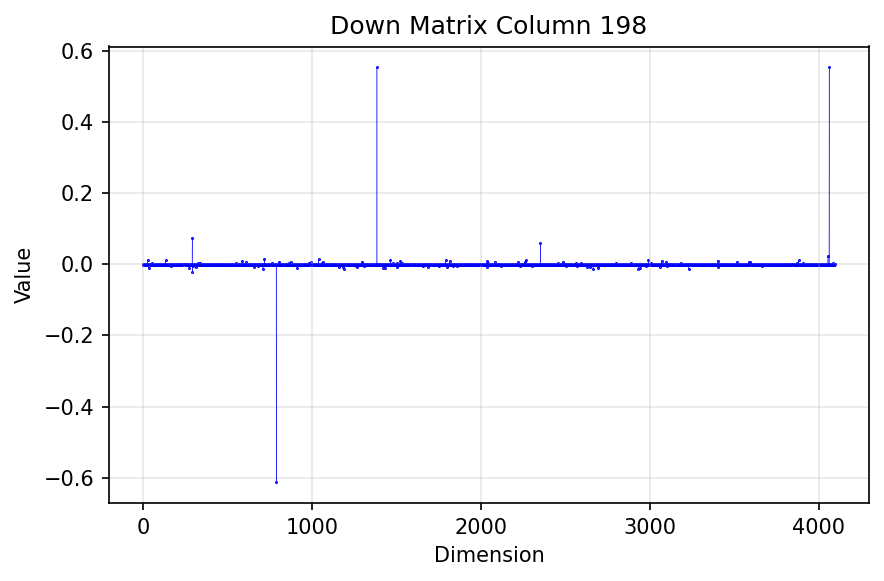}
        \caption{$\boldsymbol{\omega}_{198}$.}
    \end{subfigure}
    \hfill
    \begin{subfigure}[b]{0.48\textwidth}
        \centering
        \includegraphics[width=\linewidth]{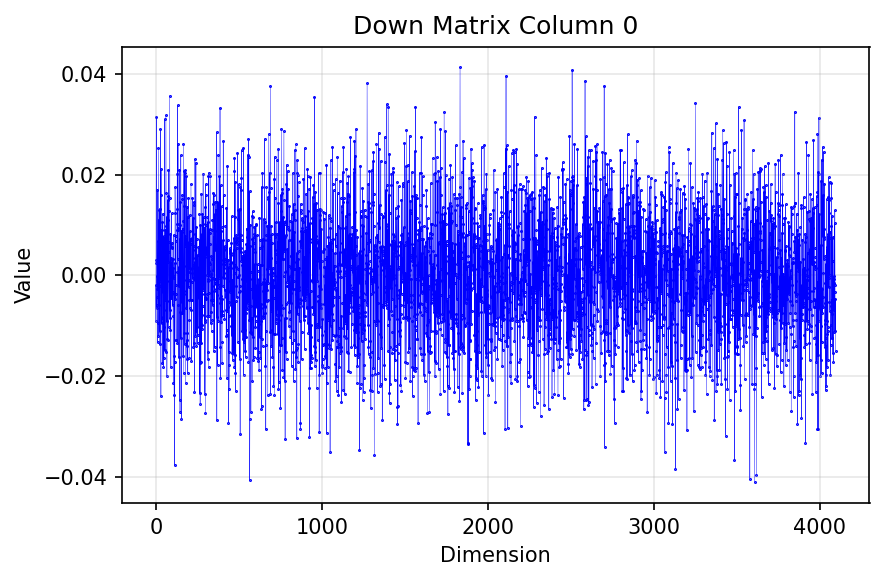}
        \caption{$\boldsymbol{\omega}_{0}$.}
    \end{subfigure}

    \caption{Results for DeepSeek-R1-Distill-Llama-8B.}
    \label{fig:deepseek-llama}
\end{figure}

\end{document}